\documentclass[10pt]{article}
\usepackage[preprint]{tmlr}

\usepackage{amsmath, amssymb, amsthm, mathtools}
\usepackage{graphicx}
\usepackage{booktabs}
\usepackage{tabularx}
\usepackage{multirow}
\usepackage{array}
\usepackage{hyperref}
\hypersetup{colorlinks=true, linkcolor=blue, citecolor=blue, urlcolor=blue}
\usepackage{cleveref}
\usepackage{microtype}
\usepackage{xcolor}
\usepackage{enumitem}
\usepackage{tikz}
\usetikzlibrary{arrows.meta, positioning}
\usepackage[section]{placeins}

\newtheorem{theorem}{Theorem}[section]
\newtheorem{lemma}[theorem]{Lemma}
\newtheorem{proposition}[theorem]{Proposition}
\newtheorem{corollary}[theorem]{Corollary}

\theoremstyle{definition}

\newtheorem{example}[theorem]{Example}
\newtheorem{remark}[theorem]{Remark}

\newcommand{\Vstar}{V^\star}
\newcommand{\Vhat}{\hat V}
\newcommand{\taunbr}{\tau_{\mathrm{nbr}}}
\newcommand{\epslab}{\epsilon_{\mathrm{lab}}}
\newcommand{\etainf}{\eta_\infty} 
\newcommand{\AMLE}{\textsc{AMLE}}
\newcommand{\bdry}{\partial}

\title{Planner-Admissible Graph-PDE Value Extensions\\
for Sparse Goal-Conditioned Planning}
\author{\name Shiheng Zhang \email shzhang3@uw.edu \\
        \addr Department of Applied Mathematics, University of Washington}

\begin{document}
\maketitle

\begin{abstract}
Sparse goal-conditioned planning with few cost-to-go labels can be
cast as graph-PDE Dirichlet extension: extend labels on a small
boundary $\Gamma_g$ to unlabelled vertices so that greedy rollouts
reach the goal.  We argue the key criterion is \emph{planner
admissibility}, not pointwise error.  Within the graph
$p$-Laplacian family, harmonic averaging ($p = 2$) often fails to
preserve local greedy orderings, whereas high-$p$ and
Absolutely Minimal Lipschitz Extension (\AMLE) /
Lipschitz-extremal extensions enter a high-success regime.
\AMLE\ is the simple, analyzable $p = \infty$ endpoint of this
family, with a local midrange update and a clean comparison /
Lipschitz theory.

Our main theoretical statement is a
\textbf{planner-admissibility certificate}: under the
operational argmin-$Q$ planner, if at every state visited by
the surrogate-greedy rollout the local value error stays below
half the true action gap, then the rollout reaches the goal
(Theorem~\ref{thm:planner-admissibility}).  \AMLE\ instantiates
the certificate through a comparison-principle fill-distance
bound (Corollary~\ref{cor:amle-admissibility}); harmonic can
violate it because its neighbour rankings are
boundary-label-weighted hitting-probability rankings, not
shortest-path rankings (Lemma~\ref{lem:harmonic-measure},
Lemma~\ref{lem:harmonic-anti-admissibility}).  All $p$-Laplacian
endpoints share a structural maximum principle
(Lemma~\ref{lem:3.6.3}): no strict interior sinks across
$p \in [2, \infty]$.  An operator-level identity
(Proposition~\ref{prop:3.7}) shows that
shortest-path distance is locally \AMLE-compatible on the
geodesically extendable interior, while harmonic-compatibility
requires a non-generic degree-balance condition
$n_+ = n_-$ (Corollary~\ref{cor:harmonic-residual}).  A refinement-stable
seven-node construction is a positive-margin witness for a local
harmonic-ordering failure mode that persists under graph
subdivision.

On $120$ AntMaze layout-derived graph
configurations~\citep{fu2020d4rl}, aggregate rollout success is
$0.584$ for $p = 2$, $0.903$ for $p = 4$, $0.973$ for $p = 8$,
$0.982$ for a fixed-budget $p = 16$ solver, and $0.970$ for
\AMLE.  Large finite-$p$ rows are not used for exact endpoint
ranking because solver certification is incomplete.  On the
rollout-weighted decision scope, \AMLE\ reduces the
low neighbour-Kendall-$\taunbr$ tail from $0.064$ to $0.015$ and
the true cost-to-go gap of the surrogate-chosen action from
$0.049$ to $0.006$.
\end{abstract}

%
%

\section{Introduction}
\label{sec:intro}

Goal-conditioned reinforcement learning (GCRL) supports agents
that reach designated states by consuming a goal-conditioned
value function
$\Vstar(s, g)$ on a state-transition graph
\citep{schaul2015universal, andrychowicz2017hindsight,
eysenbach2019sorb}.  In practice this value function is observed
at only a small subset of state-goal pairs --- a sparse boundary
$\Gamma_g \subset V$ of cost-to-go labels --- and must be
extended to unlabelled interior vertices before the agent can
plan.  Many extension principles are available: harmonic averaging
\citep{zhu2003harmonic}, smoothness-regularised graph
$p$-Laplacians \citep{kyng2015algorithms, calder2018game}, and
Lipschitz-extremal completions \citep{aronsson1967extension,
peres2006tug}.  Existing analyses of value approximation justify
greedy policies through action-value ordering or weighted
$L^p$ loss bounds \citep{singh1994upper, munos2003error}, but
graph semi-supervised learning objectives optimise extension
smoothness or pointwise fit.  Our criterion is the induced graph
planner: which extension principle preserves the local action
ordering on which one-step decisions depend?

We frame this as a graph partial differential equation
(graph-PDE) Dirichlet extension problem.  The
state-transition graph $G = (V, E, w)$ has positive edge weights
$w(s, y)$ encoding transition costs; the boundary $\Gamma_g
\subset V$ carries the observed labels with $g \in \Gamma_g$ as
the goal; the surrogate $\Vhat$ extends these labels to the
interior $V \setminus \Gamma_g$.  The planner used throughout this
paper, matching the graph-search rule of
Search-on-Replay-Buffer-style methods~\citep{eysenbach2019sorb},
is the operational argmin-$Q$ rule
\[
  T_{\Vhat, g}(s) \;=\; \arg\min_{y \,\sim\, s}\,
  \bigl[w(s, y) + \Vhat(y, g)\bigr],
\]
with consistent tie-breaking and absorbing $g$.  Within the graph
$p$-Laplacian family, the natural endpoints are harmonic averaging
at $p = 2$ and the Absolutely Minimal Lipschitz Extension (\AMLE)
at $p = \infty$, with intermediate $p$ interpolating between
them; the question is which members of the family preserve the
action ordering well enough for the argmin-$Q$ rollout to reach
the goal from interior starts.

Our main theoretical result is a
\textbf{planner-admissibility certificate}
(Theorem~\ref{thm:planner-admissibility}): if local surrogate
error stays below half the true action gap along the realised
rollout, the greedy rollout reaches the goal.  \AMLE\
instantiates the certificate via a fill-distance comparison
bound (Corollary~\ref{cor:amle-admissibility}); harmonic can
violate it because its rankings are harmonic-measure rankings
rather than shortest-path rankings
(Lemma~\ref{lem:harmonic-anti-admissibility}).  All
$p$-Laplacian endpoints share a structural maximum principle
(Lemma~\ref{lem:3.6.3}): no strict interior local extrema across
$p \in [2, \infty]$.  The \AMLE-vs-harmonic separator therefore
lives in the local action-ordering layer.

Empirically, on $120$ paired AntMaze layout-derived
configurations~\citep{fu2020d4rl} ($61{,}440$ rollouts per
method), aggregate rollout success is $0.584 \pm 0.230$ for
$p = 2$ and $\mathbf{0.970 \pm 0.061}$ for $p = \infty$, a
$\mathbf{+38.6}$~pp paired lift (95\% run-bootstrap CI
$[+34.9, +42.3]$).  Intermediate $p \in \{4, 8, 16\}$ cluster
with \AMLE\ in the $0.90$--$0.98$ regime, with the few-pp
$p = 16$ vs $p = \infty$ lift solver-tolerance-bound rather
than a converged endpoint ranking.  The mechanism audit in
Section~\ref{sec:experiments}
localises the gap to operator geometry: $54.8\%$ of rollout
decisions sit in geometry where the true distance $d_g$ is
locally \AMLE-compatible but harmonic-incompatible
(Proposition~\ref{prop:3.7} and
Corollary~\ref{cor:harmonic-residual}); $99.3\%$ of harmonic
inversions concentrate there,
\AMLE\ corrects $93.1\%$ of inversions, and the per-state
half-gap test of Lemma~\ref{lem:3.6.X.2} (the local hypothesis
underlying Theorem~\ref{thm:planner-admissibility}) certifies
$67.0\%$ of inversion events.  The ordering and mechanism audits
are decision-scope diagnostics, not direct upper bounds on
$\phi(\Vhat)$.

The local mechanism is illustrated by a seven-node graph $G_7$
(Example~\ref{ex:3.X}) with sparse-label boundary
$\Gamma_g = \{0, 7\}$ and labels $Y_g(0) = 0$, $Y_g(7) = 3$.
At decision state $s = 4$ the true greedy neighbour is $3$,
but harmonic assigns $\hat u_2(1) = 36/29 < \hat u_2(3) = 39/29$
(\emph{wrong}), whereas \AMLE\ assigns $\hat V_\infty(3) = 1 <
\hat V_\infty(1) = 4/3$ (\emph{correct}); the inversion persists
under uniform graph subdivision
(Corollary~\ref{cor:subdivision-equivariance}).  An adversarial
$4 \times 4$ subgraph search complements this construction with
mirror cases where harmonic is greedy-perfect and \AMLE\
degenerates on a plateau, so the seven-node example illustrates
a mechanism rather than a global comparison.

\paragraph{Contributions.}
We formulate sparse GCRL value completion as a planner-admissible
graph-PDE extension problem and read the surrogate choice as a
$p$-family geometry question.  Specifically:
\begin{itemize}[itemsep=2pt, topsep=2pt, parsep=0pt]
\item We prove a local action-gap certificate
  (Theorem~\ref{thm:planner-admissibility}) and its \AMLE\
  fill-distance instantiation
  (Corollary~\ref{cor:amle-admissibility}) translating
  sparse-label density into greedy-rollout success.
\item We establish harmonic anti-admissibility through
  harmonic-measure rankings
  (Lemma~\ref{lem:harmonic-anti-admissibility}) and an
  operator-level mismatch with shortest-path distance
  (Proposition~\ref{prop:3.7},
  Corollary~\ref{cor:harmonic-residual}).
\item We validate the resulting $p$-family transition on $120$
  AntMaze graph layouts via the main rollout experiment, the
  ordering audit, the mechanism audit, and the failure-mode
  decomposition.
\item Adaptive label selection, closed-loop continuous control,
  and certified intermediate-$p$ solvers are open directions
  deferred to §\ref{sec:future}.
\end{itemize}


\section{Related work and positioning}%
\label{sec:related}%

\paragraph{Harmonic / Laplacian SSL and low-label degeneracy.}
The classical harmonic SSL approach fits a function agreeing
with labels on a boundary set and minimising quadratic Dirichlet
energy on the interior~\citep{zhu2003harmonic, zhou2004learning},
generalised by manifold and graph-Laplacian
regularisation~\citep{belkin2003laplacian, belkin2006manifold,
bengio2006label, smola2003kernels}.  At very sparse labellings the
harmonic minimiser is known to degenerate to a near-constant
spike pattern~\citep{nadler2009limit}; properly-weighted graph
Laplacians restore well-posedness~\citep{calder2020properly,
calder2023rates}.  Harmonic is the $p = 2$ endpoint of our
$p$-family; its low-label degeneracy is one structural reason
planner admissibility should not be checked through pointwise
harmonic accuracy.

\paragraph{\AMLE, $p$-Laplacian, and Lipschitz learning.}
The absolutely minimising Lipschitz extension (\AMLE) of
\citet{aronsson1967extension} coincides in the continuum with
the viscosity solution of the infinity-Laplace equation, made
precise via tug-of-war games~\citep{peres2006tug,
sheffield2010vector, manfredi2012dynamic} and characterised by
comparison principles~\citep{jensen1993uniqueness,
juutinen2006equivalence, legruyer2007amle}; convergent
numerics: \citet{oberman2005convergent}.  The graph
$p$-Laplacian interpolates harmonic ($p=2$) and \AMLE\ 
($p=\infty$); efficient algorithms, game-theoretic unification,
and consistency / convergence-rate analyses are surveyed
in~\citep{kyng2015algorithms, calder2018game, bungert2021lipschitz,
elmoataz2017game, elalaoui2016asymptotic, slepcev2019analysis,
calder2019consistency, roith2023continuum, garciatrillos2020maximum}.
Provably convergent fast $\ell_p$-regression solvers for
intermediate $p$~\citep{adil2019irls, adil2024fastlp} are
relevant to §\ref{sec:future}.

\paragraph{Eikonal viscosity solutions on networks.}
Discrete eikonal / shortest-path solvers include fast
marching~\citep{sethian1996fast}, Dijkstra-style
discretisations~\citep{tsitsiklis1995efficient}, and heat-flow
alternatives~\citep{crane2013geodesics}; viscosity theory on
metric graphs is developed
in~\citep{schieborn2011eikonal, camilli2013comparison,
achdou2013constrained}.  Proposition~\ref{prop:3.7} is a distinct
statement: shortest-path distance also satisfies the graph-\AMLE\ 
midrange identity on geodesically extendable interiors.  We do
\emph{not} claim that the sparse-label \AMLE\ extension is the
network eikonal solution.

\paragraph{GCRL and positioning.}
GCRL value learning~\citep{schaul2015universal,
andrychowicz2017hindsight, pong2018tdm} on D4RL
AntMaze~\citep{fu2020d4rl} combines with planning-aware
architectures~\citep{tamar2016vin, eysenbach2019sorb,
nasiriany2019planning} and spectral / Laplacian
viewpoints~\citep{mahadevan2007proto, dayan1993successor,
stachenfeld2017predictive}.  Existing graph-PDE / Lipschitz
work studies interpolation, regularisation, and continuum
limits but not which extension avoids greedy-planner failure
modes; existing GCRL work does not study sparse graph-label
completion.  The theme that value error matters for greedy
control only through action ordering goes
back~\citep{singh1994upper, williams1993tight, munos2003error,
munos2007performance, farahmand2010error}; we sit at this
intersection.  Our planner-admissibility theory selects within
the graph $p$-Laplacian family via a local action-margin
certificate (Theorem~\ref{thm:planner-admissibility}), with the
local action-ordering and operator-compatibility properties
(Lemma~\ref{lem:harmonic-anti-admissibility},
Proposition~\ref{prop:3.7}) as the separating axes; the family
shares a structural maximum principle (Lemma~\ref{lem:3.6.3}).  We test on D4RL
AntMaze \emph{layout graphs}; closed-loop continuous control is
deferred (§\ref{sec:future}).

\section{Planner admissibility under the GCRL argmin-$Q$ planner}
\label{sec:admissibility}

We develop the \emph{planner-admissibility certificate}
(§\ref{sec:planner-admissibility}): a local action-margin
condition along the realised rollout sufficient for the
$\Vhat$-greedy planner to reach the goal.  It is the local
graph-planning analogue of classical value-error-to-greedy-policy
bounds~\citep{williams1993tight, singh1994upper, munos2003error},
with the global Bellman-residual norm replaced by per-state
local error $\epsilon_s$ and the discount-factor condition by a
one-step action gap $\Delta^\star_s$.
Section~\ref{sec:amle-instantiation} instantiates the certificate
for \AMLE\ via a sparse-label fill-distance bound; the remaining
subsections explain how harmonic \emph{can} violate the
certificate in planner-relevant local decisions through
harmonic-measure rather than shortest-path rankings, while both
endpoints share an interior maximum principle.

\subsection{Operational setup}
\label{sec:adm-setup}

Let $G = (V, E, w)$ be a finite, connected, undirected graph with
positive symmetric edge weights, $\bdry \subset V$ a non-empty
boundary, and $g \in \bdry$ a goal vertex.  The true cost-to-go
$\Vstar(\cdot, g) : V \to \mathbb{R}_{\ge 0}$ is the graph
shortest-path cost-to-go: $\Vstar(g, g) = 0$ and, for every
$s \in V \setminus \{g\}$,
\begin{equation}
  \Vstar(s, g) \;=\; \min_{y : y \sim s} \bigl[w(s, y) + \Vstar(y, g)\bigr]
  \label{eq:bellman-shortest-path}
\end{equation}
(the Bellman / shortest-path identity used in the proofs below).
We assume $\Vstar$ is known on $\bdry$, and a surrogate
$\Vhat : V \to \mathbb{R}$ extends the boundary condition
$\Vhat\rvert_\bdry = \Vstar\rvert_\bdry$ to the interior.

We adopt the operational argmin-$Q$ planner of the experiments
(§\ref{sec:experiments}).  Define the surrogate action value
$Q_{\Vhat, g}(s, y) := w(s, y) + \Vhat(y, g)$ on each edge
$\{s, y\} \in E$, and the deterministic successor
\[
  T_{\Vhat, g}(s) \;:=\; \arg\min_{y : y \sim s} Q_{\Vhat, g}(s, y),
\]
with a fixed consistent tie-breaking rule (lexicographic on
vertex indices throughout this paper).  The goal $g$ is
absorbing: $T_{\Vhat, g}(g) := g$.  Other boundary vertices are
\emph{not} absorbing --- the planner continues from any
$z \in \bdry \setminus \{g\}$ along the same rule.  Starting
from $s_0 \in V$, the rollout is the orbit
$s_{t + 1} = T_{\Vhat, g}(s_t)$ for $t \ge 0$.  Since $V$ is
finite and $T_{\Vhat, g}$ is deterministic, the orbit
eventually enters a finite set $\rho(s_0) \subseteq V$ that is
either $\{g\}$ (operational success) or a non-goal directed
cycle (operational failure).

The \emph{operational failure rate} is
$\phi(\Vhat) := |\{s_0 \in V : \rho(s_0) \ne \{g\}\}| /
|V \setminus \{g\}|$.

\begin{lemma}[Operational basin partition]
\label{lem:bookkeeping}
Under the deterministic argmin-$Q$ planner above on a finite
connected graph, every $s_0 \in V$ has $\rho(s_0)$ equal to
$\{g\}$ or to a directed cycle in $V \setminus \{g\}$.
Equivalently, $V$ partitions into the goal basin $\{s_0 :
\rho(s_0) = \{g\}\}$ and the operational failure set
$F(\Vhat) := \{s_0 : \rho(s_0) \ne \{g\}\}$.
\end{lemma}

\begin{proof}
The forward orbit on a finite state space with deterministic
$T_{\Vhat, g}$ eventually repeats; with $g$ absorbing, the limit
set is either $\{g\}$ or a cycle in $V \setminus \{g\}$.
\end{proof}

In Section~\ref{sec:experiments}, the main rollout experiment
records per-start outcomes as
\texttt{reached} ($\rho(s_0) = \{g\}$) or \texttt{loop}
($s_0 \in F(\Vhat)$); the \texttt{loop} share equals
$\phi(\Vhat)$.

\subsection{The planner-admissibility certificate}
\label{sec:planner-admissibility}

We first state the certificate for an arbitrary boundary-pinned
surrogate $\Vhat$; §\ref{sec:amle-instantiation} then lets the
goal-specific sparse label set $\Gamma_g$ play the role of
$\bdry$ and specialises the surrogate to the \AMLE\ extension.  The certificate is built from three
local ingredients at each non-goal state $s$: the true
Bellman-optimal neighbour set $A^\star(s, g)$, the true local
action gap $\Delta^\star_s$, and the surrogate local error
$\epsilon_s$.

Write
$N(s) := \{y \in V : y \sim s\}$ for the neighbour set of $s$
and $d(s) := |N(s)|$ for the vertex degree, and write the true
one-step action value as
\[
  Q^\star(s, y; g) \;:=\; w(s, y) + \Vstar(y, g).
\]
Define the true Bellman-optimal neighbour set at each non-goal
state $s$:
\[
  A^\star(s, g)
  \;:=\;
  \arg\min_{y \sim s}\, Q^\star(s, y; g),
\]
and the true \emph{action gap}
\[
  \Delta^\star_s
  \;:=\;
  \begin{cases}
    \min\limits_{b \in N(s) \setminus A^\star(s, g)}\,
    \bigl[Q^\star(s, b; g) - \min\limits_{a \in A^\star(s, g)} Q^\star(s, a; g)\bigr],
      & N(s) \setminus A^\star(s, g) \ne \emptyset, \\
    +\infty,
      & N(s) = A^\star(s, g).
  \end{cases}
\]
(The $+\infty$ branch is the trivial case where every neighbour
is optimal; the certificate then holds vacuously.)  For a
surrogate $\Vhat$, the local error on $N(s)$ is
\[
  \epsilon_s
  \;:=\;
  \max_{y \sim s} \bigl|\Vhat(y, g) - \Vstar(y, g)\bigr|.
\]
Edge costs cancel in the argmin-$Q$ comparison, so
$|Q_{\Vhat}(s, y) - Q^\star(s, y; g)| = |\Vhat(y, g) -
\Vstar(y, g)| \le \epsilon_s$: the value-level local error
$\epsilon_s$ directly controls Q-rankings at $s$.  This
cancellation is why the certificate below is a single
state-level condition on $\Vhat$ rather than on $Q_{\Vhat}$.

The local-error / action-gap pair $(\epsilon_s,
\Delta^\star_s)$ enters the theory through a single-state
action-ordering lemma; its top-1 conclusion is the
theorem-critical part, and a pairwise-inversion bound is a
diagnostic by-product used by the ordering-audit table in Section~\ref{sec:experiments}.

\begin{lemma}[Local action-ordering preservation]
\label{lem:3.6.X.2}
\label{lem:3.6.X}
At every non-goal $s \in V \setminus \{g\}$, if
$\epsilon_s < \Delta^\star_s / 2$ then
\[
  \arg\min_{y \sim s}\, Q_{\Vhat, g}(s, y)
  \;\subseteq\; A^\star(s, g).
\]
For the Kendall-$\tau$ diagnostic statement that follows, assume
in addition that $d(s) \ge 2$ and fix a deterministic
tie-breaking convention for true $Q^\star$-ties (tied pairs are
counted in $M_{s, g}(0)$).  The top-1 set-admissibility
conclusion above does \emph{not} require tie-free true actions.
The neighbour-Kendall-$\tau$ at $(s, g)$ defined by
$\taunbr(s, g) := 1 - \frac{4 \, I_{s, g}}{d(s)(d(s) - 1)}$
with $I_{s, g}$ the count of pairwise inversions of the
surrogate ordering relative to the true ordering also satisfies
the gap-dependent bound
\[
  \taunbr(s, g) \;\ge\; 1 - \frac{4 \, M_{s, g}(2 \epsilon_s)}{d(s)(d(s) - 1)},
\]
where $M_{s, g}(\eta) := \#\{(i, j) : i < j, |Q^\star_i - Q^\star_j| \le \eta\}$
counts small-gap pairs of neighbours.
\end{lemma}

\begin{proof}[Proof sketch]
For $a \in A^\star(s, g)$, $b \in N(s) \setminus A^\star(s, g)$,
the true gap is at least $\Delta^\star_s$ while the surrogate
perturbs each $Q$-value by at most $\epsilon_s$, so
$Q_{\Vhat, g}(s, b) - Q_{\Vhat, g}(s, a) \ge \Delta^\star_s -
2\epsilon_s > 0$.  Hence $\arg\min_y Q_{\Vhat,g}(s, y) \subseteq
A^\star(s, g)$.  The same gap argument applied pair-by-pair
yields the Kendall-$\tau$ bound.  Full proof in
Appendix~\ref{app:proofs}.
\end{proof}

Single-state action-ordering lifts to a rollout-level
admissibility certificate.

\begin{theorem}[Planner-admissibility certificate]
\label{thm:planner-admissibility}
Consider the deterministic $\Vhat$-greedy rollout
$s_0, s_1, \ldots$ on $G$ with positive edge costs and absorbing
$g$.  If for every non-goal state $s_t$ visited before the rollout
either reaches $g$ or repeats a non-goal state,
\[
  \epsilon_{s_t} \;<\; \Delta^\star_{s_t} / 2,
\]
then $T_{\Vhat, g}(s_t) \in A^\star(s_t, g)$ at every such step,
$\Vstar(s_t, g)$ strictly decreases along the rollout, and the
rollout reaches $g$ in finitely many steps.
\end{theorem}

\begin{proof}[Proof sketch]
Lemma~\ref{lem:3.6.X.2} at each $s_t$ gives $s_{t+1} \in
A^\star(s_t, g)$, so $\Vstar(s_{t+1}, g) = \Vstar(s_t, g) -
w(s_t, s_{t+1})$ strictly decreases; with positive edge costs,
the strict-descent sequence reaches $\Vstar = 0$ (i.e., $g$) in
at most $\lceil \Vstar(s_0, g) / w_{\min} \rceil$ steps, where
$w_{\min} := \min_{\{s, y\} \in E} w(s, y) > 0$.  Finiteness
excludes non-goal cycles.
\end{proof}

Theorem~\ref{thm:planner-admissibility} is the central
operational certificate: \emph{rollout success follows from a
local action-margin condition along the actual trajectory}.  Its
single-state version Lemma~\ref{lem:3.6.X.2} is the verification
tool used to certify the \AMLE\ endpoint and to flag harmonic
failures in subsequent subsections.

The ordering audit (Kendall-$\tau$) and mechanism audit (local action margin) are
decision-scope diagnostics for the certificate's local
ingredients, not direct upper bounds on $\phi(\Vhat)$ --- which
the main rollout experiment measures directly.  For a start distribution $\mu$ on $V$,
writing $\phi_\mu(\Vhat) := \Pr_{s_0 \sim \mu}[\rho(s_0) \ne
\{g\}]$ for the corresponding failure probability, the
path-existential inequality $\phi_\mu(\Vhat) \le \Pr_{s_0 \sim
\mu}[\exists s_t : \epsilon_{s_t} \ge \Delta^\star_{s_t}/2]$
and a bad-tail mass diagnostic
(Corollary~\ref{cor:3.6.X.4}) are recorded in
Appendix~\ref{app:proofs}.

The certificate's content is now fully laid out at the
operational level.  We turn next to its first concrete
instantiation: the \AMLE\ endpoint of the graph $p$-Laplacian
family, where a comparison-principle fill-distance bound
controls $\epsilon_s$ directly from sparse-label density.

\subsection{\AMLE\ instantiation of the certificate}
\label{sec:amle-instantiation}

We specialise to unit-cost graphs ($w \equiv 1$) with boundary
the goal-specific labelled set $\Gamma_g \subset V$,
$g \in \Gamma_g$, and observed labels $Y_g : \Gamma_g \to
\mathbb{R}$ satisfying $\|Y_g - \Vstar(\cdot, g)\|_{\infty,
\Gamma_g} \le \epslab$.  The \AMLE\ extension $\hat V_\infty :=
\mathcal{A}_{\Gamma_g}(Y_g)$ is the unique boundary-pinned fixed
point of the midrange operator
$\mathcal{A}[u](x) := \tfrac{1}{2}(\min_{y \sim x} u(y) +
\max_{y \sim x} u(y))$ on $V \setminus \Gamma_g$
(\citealp{sheffield2010vector, kyng2015algorithms,
aronsson1967extension, jensen1993uniqueness,
juutinen2006equivalence}); it controls extremal slopes rather
than averaging values.  Under $w \equiv 1$ the planner reduces
to $T_{\hat V_\infty, g}(s) = \arg\min_{y \sim s} \hat
V_\infty(y)$, so the admissibility conditions of
§\ref{sec:planner-admissibility} are value-level conditions on
$\hat V_\infty$ alone.

\begin{lemma}[\AMLE\ local extension error]
\label{lem:3.6.X.1}
Let $A \subseteq V \setminus \Gamma_g$.  Define the \AMLE\ residual of
the true cost-to-go on $A$:
\[
  \etainf(\Vstar; A, \Gamma_g) \;:=\;
  \bigl\| \mathcal{A}_{\Gamma_g}(\Vstar(\cdot, g)\rvert_{\Gamma_g}) - \Vstar(\cdot, g) \bigr\|_{\infty, A}.
\]
Then
\[
  \bigl\| \hat V_\infty(\cdot, g) - \Vstar(\cdot, g) \bigr\|_{\infty, A}
  \;\le\; \epslab + \etainf(\Vstar; A, \Gamma_g).
\]
If, in addition, $\Vstar(\cdot, g)$ is $L_g$-Lipschitz on
$A \cup \Gamma_g$ with fill distance
$h_{\Gamma_g}(A) := \max_{x \in A} \min_{z \in \Gamma_g} d_G(x, z)$,
then $\etainf(\Vstar; A, \Gamma_g) \le 2 L_g \, h_{\Gamma_g}(A)$, and
in particular
\[
  \bigl\| \hat V_\infty - \Vstar \bigr\|_{\infty, A}
  \;\le\; \epslab + 2 L_g \, h_{\Gamma_g}(A).
\]
\end{lemma}

\begin{proof}
See Appendix~\ref{app:proofs} (proof of Lemma~\ref{lem:3.6.X.1}).
\end{proof}

Combining Lemma~\ref{lem:3.6.X.1} (with $A = N(s_t)$) and
Theorem~\ref{thm:planner-admissibility} yields the \AMLE\
admissibility certificate.

\begin{corollary}[\AMLE\ planner-admissibility certificate]
\label{cor:amle-admissibility}
If along the \AMLE-greedy rollout from $s_0$,
\[
  \epslab + 2 L_g \, h_{\Gamma_g}(N(s_t))
  \;<\; \Delta^\star_{s_t} / 2
\]
for every non-goal state $s_t$ visited before termination, then
the \AMLE-greedy rollout from $s_0$ reaches $g$.
\end{corollary}

\begin{proof}
Lemma~\ref{lem:3.6.X.1} applied with $A = N(s_t)$ gives
$\epsilon_{s_t}^{\AMLE} \le \epslab + 2 L_g h_{\Gamma_g}(N(s_t))$,
which under the hypothesis is strictly below
$\Delta^\star_{s_t}/2$.
Theorem~\ref{thm:planner-admissibility} concludes.
\end{proof}

Corollary~\ref{cor:amle-admissibility} is the
\emph{sparse-label-density-to-greedy-correctness} chain for
\AMLE: small enough $\epslab + 2 L_g h_{\Gamma_g}$ relative to
the local action gap along the rollout implies success.  The
bound is per-state along the realised rollout; a global-uniform
variant (Corollary~\ref{cor:3.6.X.3}) is in
Appendix~\ref{app:proofs}.  The harmonic certificate is not in general satisfied.  We first
record the $p$-Laplacian maximum principle shared by all
endpoints (§\ref{sec:fam-max-principle}), then develop the
harmonic-side mechanism that violates the certificate
(§\ref{sec:harmonic-anti-admissibility}) and the operator-level
explanation (§\ref{sec:operator-compat}).

\subsection{Maximum principle for the \texorpdfstring{$p$}{p}-Laplacian family}
\label{sec:fam-max-principle}

Exact $p$-Laplacian Dirichlet solutions admit no strict interior
local extrema, a structural property shared across
$p \in [2, \infty]$ and inherited in particular by both harmonic
and \AMLE\ (Lemma~\ref{lem:3.6.3}).

Define $M(u) := \{m \in V \setminus \Gamma_g : u(m) < u(s')
\text{ for every } s' \sim m\}$ as the strict-interior-min set of
any $u : V \to \mathbb{R}$.

\begin{lemma}[\texorpdfstring{$p$-Laplacian}{p-Laplacian} maximum principle: no strict interior extrema]
\label{lem:3.6.3}
Let $u : V \to \mathbb{R}$ be an exact interior solution, with
pinned Dirichlet boundary $u\rvert_{\Gamma_g} = Y_g$, of any one of
the three graph $p$-Laplacian completion problems:
\begin{enumerate}[label=(\roman*)]\setlength\itemsep{0pt}
\item harmonic, $p = 2$:
$\sum_{y \sim x} (u(y) - u(x)) = 0$;
\item finite graph $p$-Laplacian, $2 < p < \infty$:
$\sum_{y \sim x} |u(y) - u(x)|^{p - 2} (u(y) - u(x)) = 0$;
\item \AMLE\ / $p = \infty$:
$u(x) = \tfrac{1}{2}\bigl(\min_{y \sim x} u(y) + \max_{y \sim x} u(y)\bigr)$.
\end{enumerate}
Then $u$ has no strict interior local minimum and no strict
interior local maximum.  In particular $M(u) = \emptyset$.
\end{lemma}

\begin{proof}
See Appendix~\ref{app:proofs} (proof of Lemma~\ref{lem:3.6.3}).
\end{proof}

\begin{remark}[Operational consequence: interior cycles are plateau cycles]
\label{rem:3.6.3.2}
Let $u$ be an exact solution of any of the three Dirichlet
problems in Lemma~\ref{lem:3.6.3} on a finite connected
unweighted graph $G$.  At every interior $s \in V \setminus
\Gamma_g$ there exists $y \sim s$ with $u(y) \le u(s)$ (else
$s \in M(u)$), so the argmin-$Q$ planner satisfies
$u(T_{u, g}(s)) \le u(s)$: $u$ is monotone non-increasing along
interior rollout segments.  A directed cycle contained in
$V \setminus \Gamma_g$ therefore forces $u$ to be constant on
the cycle (plateau cycle), and no strict-decrease interior cycle
exists for any $p \in \{2\} \cup (2, \infty) \cup \{\infty\}$.
Cycles touching $\Gamma_g \setminus \{g\}$ are not controlled
by this interior argument --- pinned boundary values are
unaffected by the max principle, so $T_{u, g}(z)$ for
$z \in \Gamma_g \setminus \{g\}$ may increase $u$ --- and are
separately tracked by the \emph{boundary-touching cycle basin}
\[
  C^\Gamma(\Vhat) \;:=\; \{s_0 \in V : \rho(s_0) \cap (\Gamma_g
  \setminus \{g\}) \ne \emptyset\}\;\;\subseteq\;\;F(\Vhat),
\]
the sub-class of operational failures whose limit cycle passes
through some labelled non-goal vertex.  The empirical failure-mode decomposition of §\ref{sec:exp-failure-decomp} shows that this
boundary-touching subclass is aggregate-dominant under both
surrogates.  The family separator therefore lies on the
local-ordering and operator-compatibility axes of subsequent
subsections.
\end{remark}

We next turn to the local-ordering mechanism by which harmonic
can violate the certificate in planner-relevant decisions:
boundary-label-weighted hitting probabilities
(§\ref{sec:harmonic-anti-admissibility}), with the
operator-level reason (§\ref{sec:operator-compat}) that
shortest-path distance is locally compatible with the \AMLE\
midrange but not with harmonic averaging unless
degree-balanced.

\paragraph{Finite-sweep refinement.}
For finite-iteration \AMLE\ solvers, a residual--margin bound
(Lemma~\ref{lem:3.6.3.B} in Appendix~\ref{app:proofs}) controls
the strict-min depth at any approximate interior local minimum
by the midrange residual; this is the \AMLE-side numerical
refinement used by the \AMLE\ iteration audit.

\subsection{Harmonic anti-admissibility certificate}
\label{sec:harmonic-anti-admissibility}

The harmonic endpoint admits the classical random-walk
representation from harmonic graph SSL~\citep{zhu2003harmonic},
which locates where its neighbour ranking can disagree with the
shortest-path ranking.

\begin{lemma}[Harmonic-measure representation of neighbour rankings]
\label{lem:harmonic-measure}
Let $h$ be the harmonic extension of $y : \bdry \to \mathbb{R}$ on
a connected unweighted graph.  With $(X_t)$ the simple random
walk and $\omega_x(z) := \mathbb{P}_x(X_{\sigma_\bdry} = z)$ the
harmonic measure (where $\sigma_\bdry := \inf\{t \ge 0 : X_t \in
\bdry\}$ is the boundary-hitting time; optional stopping gives
$h(x) = \sum_z \omega_x(z) y_z$), for any decision state $s$
with neighbours $a, b \sim s$,
\[
  h(a) - h(b) \;=\; \sum_{z \in \bdry} \bigl(\omega_a(z) - \omega_b(z)\bigr) y_z.
\]
\end{lemma}

For the goal-specific labelled boundary $\Gamma_g$ with labels
$Y_g$, write $\hat u_2$ for the harmonic ($p = 2$) extension of
$Y_g$ to $V$ and $\omega_v(z) := \mathbb{P}_v(X_{\sigma_{\Gamma_g}}
= z)$ for the harmonic measure from $v$ on $\Gamma_g$ (the
instantiation of Lemma~\ref{lem:harmonic-measure} with
$\bdry = \Gamma_g$).

\begin{lemma}[Harmonic local anti-admissibility certificate]
\label{lem:harmonic-anti-admissibility}
Fix a decision state $s$ with $A^\star(s, g) \subseteq N(s)$
the true Bellman-optimal neighbour set, and let
$b \in N(s) \setminus A^\star(s, g)$ be any strictly suboptimal
competitor.  The harmonic-greedy planner at $s$ strictly prefers
$b$ over every $a \in A^\star(s, g)$ if and only if
\[
  w(s, b) + \hat u_2(b)
  \;<\;
  w(s, a) + \hat u_2(a)
  \qquad
  \text{for every } a \in A^\star(s, g),
\]
equivalently, via Lemma~\ref{lem:harmonic-measure},
\[
  w(s, b) - w(s, a)
  \;+\;
  \sum_{z \in \Gamma_g} \bigl(\omega_b(z) - \omega_a(z)\bigr)
  \, Y_g(z) \;<\; 0
  \qquad
  \text{for every } a \in A^\star(s, g).
\]
On unit-cost graphs ($w \equiv 1$), this reduces to
$\sum_z (\omega_b(z) - \omega_a(z)) Y_g(z) < 0$ for every
$a \in A^\star(s, g)$.
\end{lemma}

\begin{proof}
Direct: $Q_{\hat u_2,g}(s,b) < \min_{a \in A^\star} Q_{\hat
u_2,g}(s,a)$ is the displayed inequality; substitute $\hat
u_2(v) = \sum_z \omega_v(z) Y_g(z)$.
\end{proof}

Harmonic ranking at $s$ is determined by differences in boundary
\emph{hitting} probabilities against boundary labels --- a
random-walk averaging ordering, not a shortest-path ordering.  A
truly shorter branch may have larger hitting probability on a
high-cost label and so be ranked higher; harmonic averaging
therefore does not by itself guarantee planner-admissible local
ordering.  \AMLE\ controls extremal slopes rather than averaging
hitting probabilities, which is what enables the local sup-error
chain through Lemma~\ref{lem:3.6.X.1} and
Corollary~\ref{cor:amle-admissibility}.

\subsection{Operator-level compatibility with shortest-path distance}
\label{sec:operator-compat}

Write $d_g(x) := d_G(x, g)$ for the \emph{unweighted}
shortest-path distance (this subsection works on the unit-cost
graph metric, consistent with the \AMLE-instantiation scope of
§\ref{sec:amle-instantiation}), and call $x$ \emph{geodesically
extendable away from $g$} if some $y_+ \sim x$ has $d_g(y_+) =
d_g(x) + 1$; let $V^{\mathrm{ext}}_g$ denote the set of such
vertices.

\begin{proposition}[Distance is graph-\AMLE\ on the extendable interior]
\label{prop:3.7}
For any $x \in V$ with $x \ne g$:
\begin{enumerate}[label=(\alph*), itemsep=0pt, topsep=2pt]
\item $\min_{y \sim x} d_g(y) = d_g(x) - 1$;
\item $\max_{y \sim x} d_g(y) \in \{d_g(x) - 1, d_g(x), d_g(x) + 1\}$;
\item $d_g$ satisfies the \AMLE\ midrange identity
$d_g(x) = \mathcal{A}[d_g](x)$ if and only if
$\max_{y \sim x} d_g(y) = d_g(x) + 1$, equivalently iff
$x \in V^{\mathrm{ext}}_g$.
\end{enumerate}
In particular, at every $x \in V^{\mathrm{ext}}_g$, $d_g$
satisfies both the discrete Bellman / eikonal identity
$d_g(x) = 1 + \min_{y \sim x} d_g(y)$ and the graph-\AMLE\
midrange identity $d_g(x) = \mathcal{A}[d_g](x)$.
\end{proposition}

\begin{proof}
See Appendix~\ref{app:proofs} (proof of Proposition~\ref{prop:3.7}).
\end{proof}

Proposition~\ref{prop:3.7}(c) is a fixed-point statement about
$d_g$ itself; the sparse-label \AMLE\ extension is a separate
object with approximation error controlled by
Lemma~\ref{lem:3.6.X.1}, and the global Bellman / eikonal
identity $d_g(x) = 1 + \min_y d_g(y)$ holds at every $x \ne g$
on all of $V$.  The \AMLE\ midrange compatibility of $d_g$ on
$V^{\mathrm{ext}}_g$ should not be identified with solving the
full eikonal / shortest-path problem from the sparse-label
data.

\begin{corollary}[Harmonic residual of the shortest-path distance]
\label{cor:harmonic-residual}
Partition $N(x) = N_+(x) \sqcup N_0(x) \sqcup N_-(x)$ by
$d_g(y) - d_g(x) \in \{+1, 0, -1\}$ with sizes $n_\sigma(x)$.
Then
\[
  \tfrac{1}{\deg(x)} \sum_{y \sim x} d_g(y) - d_g(x)
  \;=\; \tfrac{n_+(x) - n_-(x)}{\deg(x)}.
\]
So $d_g$ satisfies the harmonic-averaging identity at $x$ iff
$n_+(x) = n_-(x)$, whereas by Proposition~\ref{prop:3.7} it
satisfies the \AMLE\ midrange identity iff $n_+(x) \ge 1$.
\end{corollary}

\begin{proof}
By Proposition~\ref{prop:3.7}(a)-(b), $d_g(y) - d_g(x) \in
\{-1, 0, +1\}$ for every $y \sim x$, so the trichotomy partition
$N(x) = N_+ \sqcup N_0 \sqcup N_-$ is exhaustive.  Summing
$d_g(y) - d_g(x)$ over $y \sim x$ gives
$(+1) n_+(x) + (0) n_0(x) + (-1) n_-(x) = n_+(x) - n_-(x)$;
dividing by $\deg(x)$ yields the identity.
\end{proof}

The compatibility conditions differ in scope: \AMLE\ asks only
that some outward geodesic continuation exists through $x$,
whereas harmonic requires a non-generic degree-balance
$n_+(x) = n_-(x)$ that fails at junctions, wall-adjacent
vertices, and asymmetric branching points.  Where it fails,
harmonic pulls $d_g$ away from its true value by a residual of
order $|n_+ - n_-| / \deg(x)$ toward the over-represented
neighbour class; together with Lemma~\ref{lem:harmonic-measure},
this is the structural source of the \AMLE-vs-harmonic difference
on sparse-label graph-PDE planning.

\subsection{Combined local separation and mechanism witness}
\label{sec:harmonic-measure}

The two certificates combine into a local separation statement
(Proposition~\ref{prop:local-separation} in
Appendix~\ref{app:proofs}): on decision states where the
\AMLE\ local-admissibility condition
$\epsilon_s^{\AMLE} < \Delta^\star_s / 2$ and the harmonic
anti-admissibility inequality both hold, \AMLE-greedy is
correct and harmonic-greedy is wrong.  The mechanism audit reports the
rollout-weighted firing rates: harmonic inversion $5.6\%$,
\AMLE\ correction of inversions $93.1\%$, and certified
correction (clause (i) fires) $67.0\%$.  The $93.1\%$--$67.0\%$
gap is \AMLE-correctness via mechanisms outside the sufficient
local-error condition --- e.g.\ the Proposition~\ref{prop:3.7}
exact-distance match on $V^{\mathrm{ext}}_g$.

\begin{figure}[!ht]
\centering
\begin{tikzpicture}[
  >=Latex,
  every node/.style={font=\footnotesize},
  vtx/.style={circle, draw, thick, minimum size=7mm, inner sep=0pt, fill=white},
  bdry/.style={circle, draw, thick, minimum size=7mm, inner sep=0pt, fill=gray!15},
  decision/.style={circle, draw=black, very thick, minimum size=8mm, inner sep=0pt, fill=yellow!25},
  ed/.style={draw, thick},
  trueedge/.style={-Latex, line width=1.0pt, green!55!black},
  harmedge/.style={-Latex, line width=1.0pt, red!70!black,
                   dash pattern=on 4pt off 2.5pt},
  info/.style={draw, rounded corners, align=left, inner sep=5pt, fill=white, font=\footnotesize},
]

\node[bdry]     (n0) at (0,    0)     {$0$};
\node[vtx]      (n3) at (1.4,  0)     {$3$};
\node[decision] (n4) at (2.8,  0)     {$4$};
\node[bdry]     (n7) at (4.2,  0)     {$7$};
\node[vtx]      (n5) at (0.7, -1.15)  {$5$};
\node[vtx]      (n1) at (2.1, -1.15)  {$1$};
\node[vtx]      (n6) at (2.1,  1.15)  {$6$};

\draw[ed] (n0) -- (n3);
\draw[ed] (n0) -- (n5);
\draw[ed] (n5) -- (n1);
\draw[ed] (n1) -- (n4);
\draw[ed] (n3) -- (n4);
\draw[ed] (n3) -- (n6);
\draw[ed] (n6) -- (n7);
\draw[ed] (n4) -- (n7);

\node[above=2pt of n4, font=\scriptsize] {$s = 4$};
\node[below=2pt of n0, font=\scriptsize] {$g,\ Y_g(0) = 0$};
\node[below=2pt of n7, font=\scriptsize] {$Y_g(7) = 3$};

\draw[trueedge] (n4) to[bend left=15]  (n3);
\draw[harmedge] (n4) to[bend right=15] (n1);

\node[info, anchor=west, text width=6.4cm] at (5.4, 0) {%
\textbf{Decision at $s = 4$}\\[2pt]
True shortest-path:\\
\quad $d_g(3) = 1 < d_g(1) = 2$
\quad $\Rightarrow$ pick $3$\\[3pt]
Harmonic:\\
\quad $\hat u_2(1) = \tfrac{36}{29} < \hat u_2(3) = \tfrac{39}{29}$
\quad $\Rightarrow$ pick $1$ (suboptimal)\\[3pt]
\AMLE:\\
\quad $\hat V_\infty(3) = 1 < \hat V_\infty(1) = \tfrac{4}{3}$
\quad $\Rightarrow$ pick $3$ (optimal)
};

\end{tikzpicture}

\vspace{4pt}
{\footnotesize
\tikz[baseline=-0.6ex]\draw[-Latex, line width=1.0pt, green!55!black]
  (0,0) -- (10mm,0);\ true / \AMLE\ action
\qquad
\tikz[baseline=-0.6ex]\draw[-Latex, line width=1.0pt, red!70!black,
  dash pattern=on 4pt off 2.5pt] (0,0) -- (10mm,0);\ harmonic action
}

\caption{Seven-node mechanism witness.  At $s = 4$, the true
shortest-path action and \AMLE\ choose neighbour $3$, while
harmonic chooses neighbour $1$.  The right panel gives the
local inequalities that determine the choices.  This is a
positive-margin mechanism witness, not a universal-dominance
claim.}
\label{fig:g7-mechanism}
\end{figure}
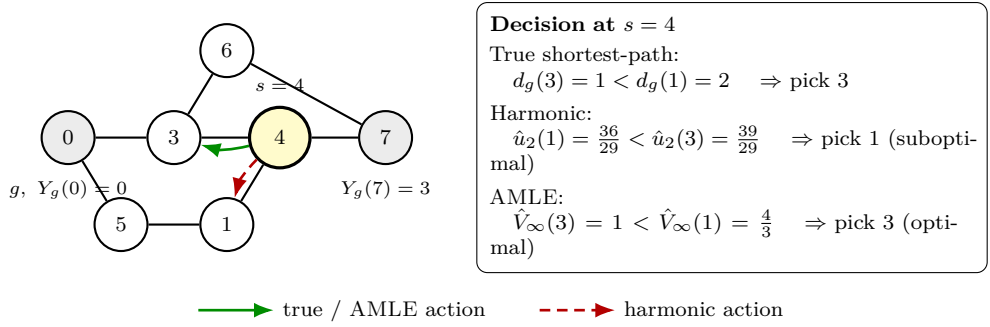

\begin{example}[Harmonic-measure mismatch on a 7-node graph]
\label{ex:3.X}
On the seven-node graph $G_7$ defined in
Appendix~\ref{app:g7-setup} (goal $g = 0$, sparse-label
boundary $\Gamma_g = \{0, 7\}$, labels $Y_g(0) = 0$, $Y_g(7) =
3$; Figure~\ref{fig:g7-mechanism}), at decision state $s = 4$
with $A^\star(4, g) = \{3\}$ and gap $\Delta^\star_4 = 1$:
\[
  \hat u_2(1) = \tfrac{36}{29} < \hat u_2(3) = \tfrac{39}{29}
  \quad\text{(harmonic, \emph{wrong})},
  \qquad
  \hat V_\infty(3) = 1 < \hat V_\infty(1) = \tfrac{4}{3}
  \quad\text{(\AMLE, \emph{correct})}.
\]
Harmonic realises clause (ii) of
Proposition~\ref{prop:local-separation} with harmonic-measure
values $\omega_1(7) - \omega_3(7) = -1/29$.  At $s = 4$, \AMLE\
is correct through the operator-level exact match $\hat
V_\infty(3) = \Vstar(3, g) = 1$ of Proposition~\ref{prop:3.7}
($3 \in V^{\mathrm{ext}}_g$) rather than through clause (i)'s
sufficient local-error condition, which gives
$\epsilon_4^{\AMLE} = 2/3 > 1/2 = \Delta^\star_4 / 2$.  This
matches the mechanism audit in Section~\ref{sec:experiments}: the
local-error test certifies $67.0\%$ of harmonic-inversion events,
while \AMLE\ corrects $93.1\%$ overall --- a $\sim 26$~pp gap
consisting of \AMLE\ corrections via operator-level mechanisms
outside the sufficient local-error condition.
Full algebra and the formal mechanism-scope remark are in
Appendix~\ref{app:proofs}.
\end{example}

The inversion of Example~\ref{ex:3.X} persists under uniform
$k$-subdivision with $\lambda_k = k$
(Corollary~\ref{cor:subdivision-equivariance} in
Appendix~\ref{app:proofs}; the analogous \AMLE\ statement is
Lemma~\ref{lem:amle-subdivision} there), so refining the graph
does not rescue harmonic.  An adversarial $4 \times 4$ subgraph
search (Appendix~\ref{app:tables-adversarial}) also exhibits
mirror cases where harmonic is greedy-perfect and \AMLE\
degenerates on interior plateaus adjacent to a high-cost
boundary vertex; the distributional question of which
surrogate wins on a graph family is settled empirically by the
§\ref{sec:experiments} aggregate.

The certificate framework and its two-sided mechanism reading
are now in place.  Section~\ref{sec:pde-family} catalogues the
canonical graph-PDE surrogates used in the experiments under
this admissibility lens, situating harmonic and \AMLE\ as the
$p = 2$ and $p = \infty$ endpoints of a larger family.

%

\section{Graph-PDE surrogate family}%
\label{sec:pde-family}%

Section~\ref{sec:admissibility} cast greedy-rollout failure as a
non-goal directed cycle in the argmin-$Q$ successor map
(Lemma~\ref{lem:bookkeeping}) and identified planner
admissibility as the design criterion for sparse-label value
extension.  We now situate each member of the graph-PDE surrogate
family under this lens.  Every exact $p$-Laplacian extension
satisfies the maximum principle of
§\ref{sec:fam-max-principle} (Lemma~\ref{lem:3.6.3}):
$M(\hat u) = \emptyset$, and interior operational cycles are
uniformly plateau cycles (Remark~\ref{rem:3.6.3.2};
boundary-touching cycles tracked separately by $C^\Gamma$).  The
differentiating axes --- local action-ordering preservation
(Lemma~\ref{lem:3.6.X}), operator-level compatibility with
shortest-path distance (Proposition~\ref{prop:3.7}), and the
harmonic-measure mismatch criterion
(Lemma~\ref{lem:harmonic-measure}) --- separate the family
members and drive the planner-admissibility verdicts in
Table~\ref{tab:pde-family}.  The eikonal row is included as an
\emph{exact-distance / oracle-adjacent reference}, not as a
sparse-label completion method: it requires full graph and
goal-boundary knowledge, and Proposition~\ref{prop:3.7}
clarifies the local \AMLE-vs-eikonal relationship without
identifying the sparse-label \AMLE\ extension with the full
eikonal solve.

\begin{table}[!ht]
\centering
\caption{Planner-admissibility taxonomy for the graph-PDE
surrogates studied in this paper.  Operational rollout
(Lemma~\ref{lem:bookkeeping}): every start either reaches $g$
or enters a non-goal cycle in $F(\Vhat)$.  Maximum-principle
column: $M(\Vhat) = \emptyset$ implies every interior cycle is
a plateau cycle (boundary-touching cycles tracked by $C^\Gamma$,
Remark~\ref{rem:3.6.3.2}).  Algorithm column: standard solver
class.}%
\label{tab:pde-family}%
\footnotesize
\setlength{\tabcolsep}{4pt}
\renewcommand{\arraystretch}{1.15}
\begin{tabular}{@{}l l p{3.4cm} p{2.7cm} p{2.2cm}@{}}
\toprule
$p$ & Equation & Admissibility status & Algorithm & Empirical \\
\midrule
$p = 2$ (harmonic) &
  $\Delta_2 u = 0$ &
  $M = \emptyset$ via Lem.~\ref{lem:3.6.3}; plateau cycles in $F$ unconstrained; anti-admissibility via Lem.~\ref{lem:harmonic-anti-admissibility} &
  $\tilde O(|E|)$ Laplacian solver &
  §5 baseline, $0.584$ \\
\addlinespace
$2 < p < \infty$ &
  $\Delta_p u = 0$ &
  $M = \emptyset$ via Lem.~\ref{lem:3.6.3}; no per-state admissibility certificate &
  IRLS / Picard / L-BFGS-B; no fast intermediate-$p$ solver &
  finite-$p$ family sweep \\
\addlinespace
$p = \infty$ (\AMLE) &
  midrange &
  $M = \emptyset$ via Lem.~\ref{lem:3.6.3}; finite-sweep via Lem.~\ref{lem:3.6.3.B}; admissibility via Cor.~\ref{cor:amle-admissibility} &
  midrange fixed-point, $O(|E|)$ / sweep &
  §5 main, $0.970$ \\
\addlinespace
eikonal (oracle) &
  $|\nabla u| = 1$ &
  exact distance (full-graph oracle) &
  Dijkstra / FMM &
  \emph{reference only} \\
\bottomrule
\end{tabular}
\end{table}

Among the sparse-label methods, only \AMLE\ admits a local
planner-admissibility certificate
(Corollary~\ref{cor:amle-admissibility}) under a fill-distance
hypothesis; harmonic has a complementary local
anti-admissibility certificate
(Lemma~\ref{lem:harmonic-anti-admissibility}); finite-$p$
inherits the shared maximum principle without an analogous
per-state ordinal guarantee.  The empirical $p = 2 \to \infty$
progression in §\ref{sec:exp-family-sweep} reads as a geometry
transition (averaging $\to$ Lipschitz-extremal), with the few-pp
$p = 16$ vs $p = \infty$ lift solver-tolerance-bound rather
than a converged endpoint ranking.

%
%
%

\section{Empirical phase diagram on D4RL AntMaze graph geometry}
\label{sec:experiments}

\subsection{Setup}
\label{sec:exp-setup}

We evaluate the planner-admissible graph-PDE family on the
\textbf{D4RL AntMaze graph geometry}~\citep{fu2020d4rl}, a
standard GCRL benchmark.  Each AntMaze layout
(medium and large) is converted to an unweighted graph $G$ by
sampling reachable cell coordinates at refinement $r \in \{4, 8,
12\}$; this gives approximate medium grids $18 \times 24$,
$36 \times 48$, $54 \times 72$, and large grids $36 \times 48$,
$72 \times 96$, $108 \times 144$.  The labelled set
$\Gamma_g \subset V$ is a sparsely sampled subset with goal-
reachability values; we sweep label fraction $\mathrm{lf} \in
\{0.02, 0.05, 0.08, 0.12\}$ over seeds $\{54, 55, 56, 57, 58\}$.

For each $(\text{maze}, r, \mathrm{lf}, \text{seed})$
configuration we fit a harmonic ($p = 2$) and an \AMLE\
($p = \infty$) surrogate to the observed labels and run the
operational argmin-$Q$ planner from $512$ uniformly sampled
start--goal evaluation pairs, for a total of $120$ paired
configurations and $61{,}440$ rollouts per method.  The main
rollout experiment records
per-pair outcomes as success at goal (\texttt{reached}) versus
rollout loop failure (\texttt{loop}).  A \texttt{loop} outcome
is exactly the operational failure event $\rho(s_0) \ne \{g\}$
of Lemma~\ref{lem:bookkeeping}; the main rollout experiment does
not directly record per-loop cycle structure, but a reconstruction
from main-rollout raw outputs (the failure-mode decomposition,
Section~\ref{sec:exp-failure-decomp}) decomposes each loop into the
interior-cycle subclass and the boundary-touching subclass
$C^\Gamma$ of Remark~\ref{rem:3.6.3.2}.

\subsection{Main rollout phase diagram}
\label{sec:exp-rollout-main}

\begin{table}[h]
\centering
\caption{Main rollout phase diagram, aggregated across all $120$ paired
$(\text{maze}, r, \mathrm{lf}, \text{seed})$ configurations.}
\label{tab:rollout-main}
\begin{tabular}{l rr r}
\toprule
method & success mean $\pm$ sd & run-bootstrap 95\% CI & loop share \\
\midrule
harmonic $p = 2$ & $0.584 \pm 0.230$ & $[0.543, 0.623]$ & $41.6\%$ \\
\AMLE\ $p = \infty$ & $\mathbf{0.970 \pm 0.061}$ & $\mathbf{[0.959, 0.980]}$ & $\mathbf{3.0\%}$ \\
\bottomrule
\end{tabular}
\end{table}

The paired lift is $\mathbf{+38.6 \pm 20.9}$~pp with
run-bootstrap 95\% CI $\mathbf{[+34.9, +42.3]}$~pp (Wilson 95\% on
\AMLE\ eval-pair success $[0.969, 0.971]$).  \AMLE\ reduces the
rollout-loop-failure share from $41.6\%$ to $3.0\%$.  This is a
direct measurement of the operational failure rate
$\phi(\Vhat)$ of Lemma~\ref{lem:bookkeeping}: $\hat V_\infty$
leaves $\sim 3\%$ of starts in non-goal cycles, against
harmonic's $\sim 42\%$.  The run-bootstrap CI shows the
aggregate lift is stable across resampled configurations.  The
$C^\Gamma$ subclass of boundary-touching cycles is decomposed
separately via the failure-mode decomposition
(Section~\ref{sec:exp-failure-decomp}).

\paragraph{Resolution and label-fraction slices.}
Finer AntMaze graph extraction amplifies the harmonic gap: at
$r \in \{4, 8, 12\}$ the lift is
$\{+14.3, +43.7, +57.8\}$~pp (harmonic
$\{0.850, 0.534, 0.368\}$ versus \AMLE\
$\{0.993, 0.971, 0.946\}$).  The harmonic gap also persists
across the label-fraction sweep $\mathrm{lf} \in \{0.02, 0.05,
0.08, 0.12\}$ with lift $\{+39.0, +44.3, +40.9, +30.1\}$~pp;
\AMLE\ saturates to $1.000$ at $\mathrm{lf} = 0.12$, consistent
with Lemma~\ref{lem:3.6.X.1}'s fill-distance bound.  Full
per-cell tables in Appendix~\ref{app:tables-rollout-main}.  The
resolution sweep is aggregate evidence complementary to (but
distinct from) the fixed-graph
Corollary~\ref{cor:subdivision-equivariance}.

\paragraph{Ordering audit.}
On the rollout-weighted decision scope, \AMLE\ reduces the
low-$\taunbr$ tail rate from $0.064$ to $0.015$, the mean true
gap of the surrogate-chosen action from $0.049$ to $0.006$, and
the disagreement rate from $0.025$ to $0.003$; top-1 agreement
rises from $97.5\%$ to $99.7\%$.  These are decision-scope
diagnostics for the Lemma~\ref{lem:3.6.X.2} local-ordering
condition on each method's rollout distribution, not upper
bounds on $\phi(\Vhat)$ (which the main rollout experiment
measures directly).
Scope-resolved tables (\texttt{eval\_rollouts} vs \texttt{all}):
Appendix~\ref{app:tables-ordering}.

\subsection{PDE-family sweep and mechanism audit}%
\label{sec:exp-family-sweep}%

\paragraph{Finite-$p$ family sweep and solver audit.}
The finite-$p$ rows support the family-level geometry
transition: on the full $120$-cell grid, rollout success is
$p = 2$: $0.584$; $p = 4$: $0.903$; $p = 8$: $0.973$; $p = 16$:
$0.982$; $p = \infty$: $0.970$.  We do not treat the few-pp
$p = 16$ lift over \AMLE\ as a converged endpoint ranking: under
the finite-$p$ sweep budget, L-BFGS-B reports convergence on
$95/10/0$ of $120$ cells at $p = 4/8/16$, and the solver audit
shows that the usable high-$p$ solvers (L-BFGS-B and fixed-point
Picard) reproduce high $p = 16$ success while disagreeing on
residual / certification; a capped Newton-CG stress test fails
at $p = 16$.  Finite high-$p$ methods are therefore strong
empirical members of the same high-success regime, while \AMLE\
remains the simple analyzable endpoint with no finite-$p$ tuning
parameter.  Full p-family sweep, solver-audit, and p-family
summary tables (including the
$p$-family low-$\taunbr$ tail trend) are in
Appendices~\ref{app:tables-p-sweep},~\ref{app:tables-solver-audit},
and~\ref{app:tables-p-family}.

\paragraph{Mechanism audit.}
Table~\ref{tab:mechanism-main} reports the rollout-weighted local
decision audit on the $120$-cell grid.  Like the ordering audit, these are
decision-scope diagnostics, not upper bounds on $\phi(\Vhat)$.
More than half of visited decision scopes lie in
\AMLE-compatible / harmonic-incompatible local geometry
($n_+(x) \ne n_-(x)$, Cor.~\ref{cor:harmonic-residual}), and
$99.3\%$ of harmonic inversions on the AntMaze rollout
distribution concentrate there.  Conditional on harmonic
inversions, \AMLE\ corrects $93.1\%$ of cases; the \AMLE\ local
certificate $\epsilon_s^{\AMLE} < \Delta^\star_s / 2$ fires on
$67.0\%$ of inversions, so the $93.1\%$--$67.0\%$ gap is the
mass of \AMLE\ corrections via mechanisms outside the
sufficient local-error condition (consistent with
Proposition~\ref{prop:3.7}'s operator-compatibility picture).
The concentration is empirical, not a structural theorem.

\begin{table}[!ht]
\centering
\caption{Mechanism-audit rollout-weighted summary (full $120$-cell
AntMaze grid).  Geometry classes from $d_g$ on the greedy
decision scope; inversion / correction rows conditioned on
non-tied true-best decisions.  \emph{Certified correction} is
$\epsilon_s^{\AMLE} < \Delta^\star_s/2$ firing
(Lem.~\ref{lem:3.6.X.2}), as fraction of harmonic-inverted
states ($20{,}135/30{,}037$).  Per-resolution table:
App.~\ref{app:tables-mechanism}.}%
\label{tab:mechanism-main}%
\small
\begin{tabular}{p{0.68\textwidth}r}
\toprule
diagnostic & value \\
\midrule
\AMLE-compatible / harmonic-incompatible local geometry & $54.8\%$ \\
harmonic inversion rate among non-tied decisions & $5.6\%$ \\
harmonic inversions in that geometry class & $99.3\%$ \\
\AMLE\ correction given harmonic inversion & $93.1\%$ \\
certified correction given harmonic inversion & $67.0\%$ \\
\bottomrule
\end{tabular}
\end{table}

\subsection{Failure-mode decomposition}%
\label{sec:exp-failure-decomp}%

The failure-mode decomposition reruns the main rollout experiment
on the same $61{,}440$ start--goal pairs and classifies each
operational failure
$\rho(s_0) \ne \{g\}$ by limit-cycle structure: the interior
subclass $F \setminus C^\Gamma$ (limit cycle in $V \setminus
\Gamma_g$) versus the boundary-touching limit-cycle subclass
$C^\Gamma$ of Remark~\ref{rem:3.6.3.2}.  The reconstruction's
total failure count matches the main rollout experiment's
\texttt{loop} count exactly on every one of the $120$
configurations.  The pooled split is
harmonic $41.6\% = 1.6\%\;\text{interior} + 40.0\%\;C^\Gamma$
($96.1\%$ of failures in $C^\Gamma$) and \AMLE\
$3.0\% = 0.34\%\;\text{interior} + 2.66\%\;C^\Gamma$ ($88.6\%$).
$C^\Gamma$ is aggregate-dominant under both surrogates,
confirming the Remark~\ref{rem:3.6.3.2} scope caveat that the
$p$-Laplacian maximum principle (Lemma~\ref{lem:3.6.3}) controls
interior cycles only.  The interior-cycle density under \AMLE\
is an order of magnitude below harmonic; the auxiliary
strict-minimum diagnostic points in the same direction (cases
with any strict interior minimum of $\hat u$ drop from $116/120$
under harmonic to $41/120$ under \AMLE), but the decomposition
does not by itself separate tolerance-driven strict sinks from plateau-cycle
structure, so the appropriate conclusion is numerical-scope
caution rather than a structural failure theorem.

\paragraph{Reference baselines and \AMLE\ iteration audit.}
\label{sec:exp-baselines}
Nearest-label Voronoi on the same sparse $\Gamma_g$ reaches only
$0.032 \pm 0.011$, $93.8$~pp below \AMLE, so PDE-style smoothing
is important in this sparse-label graph setup.  An oracle Dijkstra solve with the full-graph
$\Vstar$ reaches $1.000$ ($\sim 3$~pp above \AMLE), the
full-information upper bound.  The \AMLE\ iteration audit
($24$ cells, $\{50, \ldots, 5000\}$ midrange sweeps) shows
\AMLE\ success $0.707 \to 1.000$ monotonically with the
\AMLE-vs-harmonic gap \emph{widening} ($+30.8 \to +60.1$~pp),
ruling out an under-iterated-\AMLE\ reading.  Full tables in
Appendices~\ref{app:tables-baselines}--\ref{app:tables-amle-iter}.

These results leave the open problems of §\ref{sec:future}:
adaptive label selection, certified intermediate-$p$ solvers,
and extensions beyond the present unweighted graph extraction.


\section{Discussion, limitations, and open problems}
\label{sec:future}

\paragraph{Scope.}
The evaluation (Section~\ref{sec:experiments}) is on the
\emph{graph extraction} of D4RL AntMaze layouts, with label
fractions $\mathrm{lf} \in [0.02, 0.12]$ on unweighted unit-cost
graphs.  The full continuous-control setting --- MuJoCo Ant
dynamics, low-level locomotion, partial observability, and
closed-loop replanning --- is not evaluated.  Within the
admissibility framework itself, the planner-admissibility
certificate (Theorem~\ref{thm:planner-admissibility}) is stated
for general positive edge costs $w(s, y)$, but the \AMLE\
admissibility instantiation
(Corollary~\ref{cor:amle-admissibility}) is specialised to the
unweighted graph setting in §\ref{sec:amle-instantiation}, and
the harmonic comparison uses the unweighted graph-harmonic
representation of §\ref{sec:harmonic-anti-admissibility}
(Lemma~\ref{lem:harmonic-anti-admissibility});
weighted-edge and directed-graph extensions of the
differentiating mechanisms, and the very-sparse limit
$\mathrm{lf} \to 0$, are not tested.

\paragraph{Open problems.}
The $p = 2$ Laplacian solver and $p = \infty$ midrange iteration
are mature endpoints, but $p \in (2, \infty)$ sits in a solver
gap; recent fast $\ell_p$-regression~\citep{adil2019irls,
adil2024fastlp} could turn the family sweep into a
solver-certified endpoint comparison.  Theoretical refinements
of Lemma~\ref{lem:3.6.X} (sharper fill-distance dependence;
expected-ordinal-fidelity statements under structured action-gap
distributions), adaptive sparse-label selection, and
per-decision calibrated
certificates~\citep{vovk2005algorithmic, angelopoulos2021gentle}
are follow-on directions deferred to future work.

\bibliographystyle{tmlr}
\bibliography{bibliography}

@article{aronsson1967extension,
  title={Extension of functions satisfying {L}ipschitz conditions},
  author={Aronsson, Gunnar},
  journal={Arkiv f{\"o}r Matematik},
  volume={6},
  number={6},
  pages={551--561},
  year={1967}
}

@article{sheffield2010vector,
  title={Vector-valued optimal {L}ipschitz extensions},
  author={Sheffield, Scott and Smart, Charles K},
  journal={Communications on Pure and Applied Mathematics},
  volume={65},
  number={1},
  pages={128--154},
  year={2012},
  doi={10.1002/cpa.20391},
}

@article{peres2006tug,
  title={Tug-of-war and the infinity {L}aplacian},
  author={Peres, Yuval and Schramm, Oded and Sheffield, Scott and Wilson, David B.},
  journal={Journal of the American Mathematical Society},
  volume={22},
  number={1},
  pages={167--210},
  year={2009},
  doi={10.1090/S0894-0347-08-00606-1},
}

@inproceedings{kyng2015algorithms,
  title={Algorithms for {L}ipschitz learning on graphs},
  author={Kyng, Rasmus and Rao, Anup and Sachdeva, Sushant and Spielman, Daniel A},
  booktitle={Proceedings of the 28th Annual Conference on Learning Theory},
  series={PMLR},
  volume={40},
  pages={1190--1223},
  year={2015},
}

@article{bungert2021lipschitz,
  title={Uniform convergence rates for {L}ipschitz learning on graphs},
  author={Bungert, Leon and Calder, Jeff and Roith, Tim},
  journal={IMA Journal of Numerical Analysis},
  volume={43},
  number={4},
  pages={2445--2495},
  year={2023},
  doi={10.1093/imanum/drac048},
}

@article{schieborn2011eikonal,
  title={Viscosity solutions for {Eikonal} equations on topological networks},
  author={Schieborn, Dirk and Camilli, Fabio},
  journal={Calculus of Variations and Partial Differential Equations},
  volume={46},
  number={3-4},
  pages={671--686},
  year={2013},
  doi={10.1007/s00526-012-0498-z},
}

@inproceedings{fu2020d4rl,
  title={{D4RL}: Datasets for deep data-driven reinforcement learning},
  author={Fu, Justin and Kumar, Aviral and Nachum, Ofir and Tucker, George and Levine, Sergey},
  booktitle={arXiv preprint arXiv:2004.07219},
  year={2020}
}

@inproceedings{andrychowicz2017hindsight,
  title={Hindsight experience replay},
  author={Andrychowicz, Marcin and Wolski, Filip and Ray, Alex and Schneider, Jonas and Fong, Rachel and Welinder, Peter and McGrew, Bob and Tobin, Josh and Pieter Abbeel, OpenAI and Zaremba, Wojciech},
  booktitle={Advances in neural information processing systems},
  year={2017}
}

@inproceedings{schaul2015universal,
  title={Universal value function approximators},
  author={Schaul, Tom and Horgan, Daniel and Gregor, Karol and Silver, David},
  booktitle={International conference on machine learning},
  year={2015}
}

@inproceedings{zhu2003harmonic,
  title={Semi-supervised learning using {G}aussian fields and harmonic functions},
  author={Zhu, Xiaojin and Ghahramani, Zoubin and Lafferty, John D.},
  booktitle={Proceedings of the 20th International Conference on Machine Learning},
  year={2003},
}

@inproceedings{zhou2004learning,
  title={Learning with local and global consistency},
  author={Zhou, Dengyong and Bousquet, Olivier and Lal, Thomas Navin and Weston, Jason and Sch{\"o}lkopf, Bernhard},
  booktitle={Advances in Neural Information Processing Systems 16},
  year={2003},
}

@article{calder2018game,
  title={The game-theoretic $p$-{L}aplacian and semi-supervised learning with few labels},
  author={Calder, Jeff},
  journal={Nonlinearity},
  volume={32},
  number={1},
  pages={301--330},
  year={2018},
}

@inproceedings{elalaoui2016asymptotic,
  author = {El Alaoui, Ahmed and Cheng, Xiuyuan and Ramdas, Aaditya and Wainwright, Martin J. and Jordan, Michael I.},
  title = {Asymptotic behavior of {$\ell_p$}-based {L}aplacian regularization in semi-supervised learning},
  booktitle = {Proceedings of the 29th Annual Conference on Learning Theory},
  series = {Proceedings of Machine Learning Research},
  volume = {49},
  pages = {879--906},
  year = {2016}
}

@article{slepcev2019analysis,
  author = {Slep{\v{c}}ev, Dejan and Thorpe, Matthew},
  title = {Analysis of {$p$}-{L}aplacian regularization in semi-supervised learning},
  journal = {SIAM Journal on Mathematical Analysis},
  volume = {51},
  number = {3},
  pages = {2085--2120},
  year = {2019},
  doi = {10.1137/17M115222X}
}

@article{calder2020properly,
  author = {Calder, Jeff and Slep{\v{c}}ev, Dejan},
  title = {Properly-weighted graph {L}aplacian for semi-supervised learning},
  journal = {Applied Mathematics \& Optimization},
  volume = {82},
  number = {3},
  pages = {1111--1159},
  year = {2020},
  doi = {10.1007/s00245-019-09637-3}
}

@article{calder2023rates,
  author = {Calder, Jeff and Slep{\v{c}}ev, Dejan and Thorpe, Matthew},
  title = {Rates of convergence for {L}aplacian semi-supervised learning with low labeling rates},
  journal = {Research in the Mathematical Sciences},
  volume = {10},
  number = {1},
  year = {2023},
  doi = {10.1007/s40687-022-00371-x}
}

@article{garciatrillos2020maximum,
  author = {Garc{\'i}a Trillos, Nicol{\'a}s and Murray, Ryan W.},
  title = {A maximum principle argument for the uniform convergence of graph {L}aplacian regressors},
  journal = {SIAM Journal on Mathematics of Data Science},
  volume = {2},
  number = {3},
  pages = {705--739},
  year = {2020},
  doi = {10.1137/19M1245372}
}

@article{roith2023continuum,
  author = {Roith, Tim and Bungert, Leon},
  title = {Continuum limit of {L}ipschitz learning on graphs},
  journal = {Foundations of Computational Mathematics},
  volume = {23},
  number = {2},
  pages = {393--431},
  year = {2023},
  doi = {10.1007/s10208-022-09557-9}
}

@article{adil2024fastlp,
  author = {Adil, Deeksha and Kyng, Rasmus and Peng, Richard and Sachdeva, Sushant},
  title = {Fast algorithms for {$\ell_p$}-regression},
  journal = {Journal of the ACM},
  volume = {71},
  number = {5},
  pages = {1--45},
  year = {2024},
  doi = {10.1145/3686794}
}

@inproceedings{adil2019irls,
  author = {Adil, Deeksha and Peng, Richard and Sachdeva, Sushant},
  title = {Fast, provably convergent {IRLS} algorithm for $p$-norm linear regression},
  booktitle = {Advances in Neural Information Processing Systems 32},
  pages = {14166--14177},
  year = {2019}
}

@article{jensen1993uniqueness,
  author = {Jensen, Robert},
  title = {Uniqueness of {L}ipschitz extensions: minimizing the sup norm of the gradient},
  journal = {Archive for Rational Mechanics and Analysis},
  volume = {123},
  number = {1},
  pages = {51--74},
  year = {1993},
  doi = {10.1007/BF00386368}
}

@article{oberman2005convergent,
  author = {Oberman, Adam M.},
  title = {A convergent difference scheme for the infinity {L}aplacian: construction of absolutely minimizing {L}ipschitz extensions},
  journal = {Mathematics of Computation},
  volume = {74},
  number = {251},
  pages = {1217--1230},
  year = {2005}
}

@article{legruyer2007amle,
  author = {Le Gruyer, Erwan},
  title = {On absolutely minimizing {L}ipschitz extensions and {PDE} {$\Delta_\infty u = 0$}},
  journal = {NoDEA Nonlinear Differential Equations and Applications},
  volume = {14},
  number = {1},
  pages = {29--55},
  year = {2007}
}

@article{juutinen2006equivalence,
  author = {Juutinen, Petri and Shanmugalingam, Nageswari},
  title = {Equivalence of {AMLE}, strong {AMLE}, and comparison with cones in metric measure spaces},
  journal = {Mathematische Nachrichten},
  volume = {279},
  number = {9--10},
  pages = {1083--1098},
  year = {2006},
  doi = {10.1002/mana.200510411}
}

@article{calder2019consistency,
  author = {Calder, Jeff},
  title = {Consistency of {L}ipschitz learning with infinite unlabeled data and finite labeled data},
  journal = {SIAM Journal on Mathematics of Data Science},
  volume = {1},
  number = {4},
  pages = {780--812},
  year = {2019},
  doi = {10.1137/18M1199241}
}

@article{manfredi2012dynamic,
  author = {Manfredi, Juan J. and Parviainen, Mikko and Rossi, Julio D.},
  title = {Dynamic programming principle for tug-of-war games with noise},
  journal = {ESAIM: Control, Optimisation and Calculus of Variations},
  year = {2012},
  doi = {10.1051/cocv/2010046}
}

@article{belkin2003laplacian,
  author = {Belkin, Mikhail and Niyogi, Partha},
  title = {{L}aplacian eigenmaps for dimensionality reduction and data representation},
  journal = {Neural Computation},
  volume = {15},
  number = {6},
  pages = {1373--1396},
  year = {2003},
  doi = {10.1162/089976603321780317}
}

@article{belkin2006manifold,
  author = {Belkin, Mikhail and Niyogi, Partha and Sindhwani, Vikas},
  title = {Manifold regularization: a geometric framework for learning from labeled and unlabeled examples},
  journal = {Journal of Machine Learning Research},
  volume = {7},
  pages = {2399--2434},
  year = {2006}
}

@incollection{bengio2006label,
  author = {Bengio, Yoshua and Delalleau, Olivier and Le Roux, Nicolas},
  title = {Label propagation and quadratic criterion},
  booktitle = {Semi-Supervised Learning},
  pages = {193--216},
  publisher = {MIT Press},
  year = {2006}
}

@incollection{smola2003kernels,
  author = {Smola, Alexander J. and Kondor, Risi},
  title = {Kernels and regularization on graphs},
  booktitle = {Learning Theory and Kernel Machines},
  series = {Lecture Notes in Computer Science},
  volume = {2777},
  pages = {144--158},
  year = {2003},
  doi = {10.1007/978-3-540-45167-9_12}
}

@inproceedings{nadler2009limit,
  author = {Nadler, Boaz and Srebro, Nathan and Zhou, Xueyuan},
  title = {Semi-supervised learning with the graph {L}aplacian: the limit of infinite unlabelled data},
  booktitle = {Advances in Neural Information Processing Systems 22},
  pages = {1330--1338},
  year = {2009}
}

@article{sethian1996fast,
  author = {Sethian, James A.},
  title = {A fast marching level set method for monotonically advancing fronts},
  journal = {Proceedings of the National Academy of Sciences},
  volume = {93},
  number = {4},
  pages = {1591--1595},
  year = {1996},
  doi = {10.1073/pnas.93.4.1591}
}

@article{tsitsiklis1995efficient,
  author = {Tsitsiklis, John N.},
  title = {Efficient algorithms for globally optimal trajectories},
  journal = {IEEE Transactions on Automatic Control},
  volume = {40},
  number = {9},
  pages = {1528--1538},
  year = {1995},
  doi = {10.1109/9.412624}
}

@article{crane2013geodesics,
  author = {Crane, Keenan and Weischedel, Clarisse and Wardetzky, Max},
  title = {Geodesics in heat: a new approach to computing distance based on heat flow},
  journal = {ACM Transactions on Graphics},
  volume = {32},
  number = {5},
  pages = {1--11},
  year = {2013},
  doi = {10.1145/2516971.2516977}
}

@inproceedings{tamar2016vin,
  author = {Tamar, Aviv and Wu, Yi and Thomas, Garrett and Levine, Sergey and Abbeel, Pieter},
  title = {Value iteration networks},
  booktitle = {Advances in Neural Information Processing Systems 29},
  year = {2016}
}

@inproceedings{eysenbach2019sorb,
  author = {Eysenbach, Benjamin and Salakhutdinov, Ruslan and Levine, Sergey},
  title = {Search on the replay buffer: bridging planning and reinforcement learning},
  booktitle = {Advances in Neural Information Processing Systems 32},
  year = {2019}
}

@inproceedings{pong2018tdm,
  author = {Pong, Vitchyr H. and Gu, Shixiang and Dalal, Murtaza and Levine, Sergey},
  title = {Temporal difference models: model-free deep {RL} for model-based control},
  booktitle = {International Conference on Learning Representations},
  year = {2018}
}

@inproceedings{nasiriany2019planning,
  author = {Nasiriany, Soroush and Pong, Vitchyr H. and Lin, Steven and Levine, Sergey},
  title = {Planning with goal-conditioned policies},
  booktitle = {Advances in Neural Information Processing Systems 32},
  year = {2019}
}

@article{mahadevan2007proto,
  author = {Mahadevan, Sridhar and Maggioni, Mauro},
  title = {Proto-value functions: a {L}aplacian framework for learning representation and control in {M}arkov decision processes},
  journal = {Journal of Machine Learning Research},
  volume = {8},
  pages = {2169--2231},
  year = {2007}
}

@article{dayan1993successor,
  author = {Dayan, Peter},
  title = {Improving generalization for temporal difference learning: the successor representation},
  journal = {Neural Computation},
  volume = {5},
  number = {4},
  pages = {613--624},
  year = {1993},
  doi = {10.1162/neco.1993.5.4.613}
}

@article{stachenfeld2017predictive,
  author = {Stachenfeld, Kimberly L. and Botvinick, Matthew M. and Gershman, Samuel J.},
  title = {The hippocampus as a predictive map},
  journal = {Nature Neuroscience},
  volume = {20},
  pages = {1643--1653},
  year = {2017},
  doi = {10.1038/nn.4650}
}

@article{singh1994upper,
  author = {Singh, Satinder P. and Yee, Richard C.},
  title = {An upper bound on the loss from approximate optimal-value functions},
  journal = {Machine Learning},
  volume = {16},
  number = {3},
  pages = {227--233},
  year = {1994},
  doi = {10.1023/A:1022693225949}
}

@techreport{williams1993tight,
  author = {Williams, Ronald J. and Baird, Leemon C.},
  title = {Tight performance bounds on greedy policies based on imperfect value functions},
  institution = {Northeastern University, College of Computer Science},
  number = {NU-CCS-93-14},
  year = {1993}
}

@inproceedings{munos2003error,
  author = {Munos, R{\'e}mi},
  title = {Error bounds for approximate policy iteration},
  booktitle = {Proceedings of the 20th International Conference on Machine Learning},
  pages = {560--567},
  year = {2003}
}

@article{munos2007performance,
  author = {Munos, R{\'e}mi},
  title = {Performance bounds in {$L_p$}-norm for approximate value iteration},
  journal = {SIAM Journal on Control and Optimization},
  volume = {46},
  number = {2},
  pages = {541--561},
  year = {2007},
  doi = {10.1137/040614384}
}

@inproceedings{farahmand2010error,
  author = {Farahmand, Amir-massoud and Munos, R{\'e}mi and Szepesv{\'a}ri, Csaba},
  title = {Error propagation for approximate policy and value iteration},
  booktitle = {Advances in Neural Information Processing Systems 23},
  pages = {568--576},
  year = {2010}
}

@article{elmoataz2017game,
  author = {Elmoataz, Abderrahim and Desquesnes, Xavier and Toutain, Matthieu},
  title = {On the game $p$-{L}aplacian on weighted graphs with applications in image processing and data clustering},
  journal = {European Journal of Applied Mathematics},
  volume = {28},
  number = {6},
  year = {2017},
  doi = {10.1017/S0956792517000122}
}

@article{camilli2013comparison,
  author = {Camilli, Fabio and Marchi, Claudio},
  title = {A comparison among various notions of viscosity solution for {H}amilton--{J}acobi equations on networks},
  journal = {Journal of Mathematical Analysis and Applications},
  year = {2013},
  doi = {10.1016/j.jmaa.2013.05.015}
}

@article{achdou2013constrained,
  author = {Achdou, Yves and Camilli, Fabio and Cutr{\`i}, Alessandra and Tchou, Nicoletta},
  title = {{H}amilton--{J}acobi equations constrained on networks},
  journal = {NoDEA Nonlinear Differential Equations and Applications},
  volume = {20},
  number = {3},
  pages = {413--445},
  year = {2013},
  doi = {10.1007/s00030-012-0158-1}
}

@book{vovk2005algorithmic,
  author = {Vovk, Vladimir and Gammerman, Alex and Shafer, Glenn},
  title = {Algorithmic learning in a random world},
  publisher = {Springer},
  year = {2005},
  doi = {10.1007/b106715}
}

@article{angelopoulos2021gentle,
  author = {Angelopoulos, Anastasios N. and Bates, Stephen},
  title = {Conformal prediction: a gentle introduction},
  journal = {Foundations and Trends in Machine Learning},
  volume = {16},
  number = {4},
  pages = {494--591},
  year = {2023},
  publisher = {Now Publishers},
  doi = {10.1561/2200000101}
}

\appendix
\section{Audit tables for the empirical claims}%
\label{app:tables}%

This appendix collects the compact tables behind the empirical
claims used in Section~\ref{sec:experiments}.  The entries below
are migrated from the locked main rollout, ordering-audit,
subdivision-verification, finite-$p$, baseline, solver,
adversarial-search, mechanism-audit, and failure-decomposition
artifacts.  File-level provenance is recorded in the supplement
manifest.

\subsection{Main rollout per-configuration AntMaze grid}%
\label{app:tables-rollout-main}%

Tables~\ref{tab:rollout-medium-grid} and~\ref{tab:rollout-large-grid} expand
the main rollout phase diagram from Section~\ref{sec:exp-rollout-main}.  Each row
aggregates five seeds, each seed contains $512$ start/goal rollout
pairs, and the methods are paired on the same graph and evaluation
pairs.  The lift column is \AMLE\ success minus harmonic success in
percentage points.  Loop columns report the method-specific fraction
of rollout outcomes recorded as \texttt{loop}.

\begin{table}[!htbp]
\centering
\caption{Main rollout per-configuration grid for AntMaze medium graph geometry
(five seeds per row; $512$ rollouts per seed).}%
\label{tab:rollout-medium-grid}%
\small
\begin{tabular}{ccrrrrr}
\toprule
$r$ & $\mathrm{lf}$ & harmonic success & \AMLE\ success & lift (pp) & harmonic loop & \AMLE\ loop \\
\midrule
$4$ & $0.02$ & $0.849$ & $0.977$ & $+12.8$ & $0.151$ & $0.023$ \\
$4$ & $0.05$ & $0.860$ & $0.996$ & $+13.6$ & $0.140$ & $0.004$ \\
$4$ & $0.08$ & $0.802$ & $0.996$ & $+19.4$ & $0.198$ & $0.004$ \\
$4$ & $0.12$ & $0.876$ & $1.000$ & $+12.4$ & $0.124$ & $0.000$ \\
$8$ & $0.02$ & $0.533$ & $0.875$ & $+34.1$ & $0.467$ & $0.125$ \\
$8$ & $0.05$ & $0.547$ & $1.000$ & $+45.3$ & $0.453$ & $0.000$ \\
$8$ & $0.08$ & $0.580$ & $0.982$ & $+40.2$ & $0.420$ & $0.018$ \\
$8$ & $0.12$ & $0.681$ & $1.000$ & $+31.9$ & $0.319$ & $0.000$ \\
$12$ & $0.02$ & $0.231$ & $0.794$ & $+56.3$ & $0.769$ & $0.206$ \\
$12$ & $0.05$ & $0.306$ & $0.973$ & $+66.7$ & $0.694$ & $0.027$ \\
$12$ & $0.08$ & $0.448$ & $0.998$ & $+55.0$ & $0.552$ & $0.002$ \\
$12$ & $0.12$ & $0.587$ & $1.000$ & $+41.3$ & $0.413$ & $0.000$ \\
\bottomrule
\end{tabular}
\end{table}

\begin{table}[!htbp]
\centering
\caption{Main rollout per-configuration grid for AntMaze large graph geometry
(five seeds per row; $512$ rollouts per seed).}%
\label{tab:rollout-large-grid}%
\small
\begin{tabular}{ccrrrrr}
\toprule
$r$ & $\mathrm{lf}$ & harmonic success & \AMLE\ success & lift (pp) & harmonic loop & \AMLE\ loop \\
\midrule
$4$ & $0.02$ & $0.831$ & $0.979$ & $+14.7$ & $0.169$ & $0.021$ \\
$4$ & $0.05$ & $0.850$ & $0.993$ & $+14.4$ & $0.150$ & $0.007$ \\
$4$ & $0.08$ & $0.834$ & $1.000$ & $+16.6$ & $0.166$ & $0.000$ \\
$4$ & $0.12$ & $0.895$ & $1.000$ & $+10.5$ & $0.105$ & $0.000$ \\
$8$ & $0.02$ & $0.458$ & $0.932$ & $+47.3$ & $0.542$ & $0.068$ \\
$8$ & $0.05$ & $0.412$ & $0.982$ & $+57.0$ & $0.588$ & $0.018$ \\
$8$ & $0.08$ & $0.468$ & $1.000$ & $+53.2$ & $0.532$ & $0.000$ \\
$8$ & $0.12$ & $0.591$ & $1.000$ & $+40.9$ & $0.409$ & $0.000$ \\
$12$ & $0.02$ & $0.167$ & $0.854$ & $+68.7$ & $0.833$ & $0.146$ \\
$12$ & $0.05$ & $0.261$ & $0.952$ & $+69.1$ & $0.739$ & $0.048$ \\
$12$ & $0.08$ & $0.386$ & $0.996$ & $+61.0$ & $0.614$ & $0.004$ \\
$12$ & $0.12$ & $0.561$ & $1.000$ & $+43.9$ & $0.439$ & $0.000$ \\
\bottomrule
\end{tabular}
\end{table}

\subsection{Ordering audit by decision scope}%
\label{app:tables-ordering}%

Table~\ref{tab:ordering-scope-grid} reports the neighbour
Kendall-$\taunbr$ ordering audit under the two decision scopes
summarised in Section~\ref{sec:experiments}.  The \texttt{all}
scope evaluates local ordering over all sampled neighbour
decisions in the reconstructed graph cases.  The
\texttt{eval\_rollouts} scope weights decisions by the greedy
rollouts used in the main rollout experiment.

\begin{table}[!htbp]
\centering
\caption{Ordering audit by decision scope.  Deltas are
\AMLE\ minus harmonic, so negative values are favorable for
\texttt{tau\_lt\_05\_rate}, \texttt{mean\_beta\_true\_gap}, and
\texttt{positive\_gap\_rate}; positive values are favorable for
\texttt{best\_agree\_rate} and \texttt{tau\_mean}.}%
\label{tab:ordering-scope-grid}%
\small
\begin{tabular}{llrrr}
\toprule
scope & metric & harmonic & \AMLE\ & delta \\
\midrule
\texttt{all} & \texttt{tau\_lt\_05\_rate} & $0.0986$ & $0.0747$ & $-0.0239$ \\
\texttt{all} & \texttt{best\_agree\_rate} & $0.9703$ & $0.9831$ & $+0.0128$ \\
\texttt{all} & \texttt{mean\_beta\_true\_gap} & $0.0594$ & $0.0338$ & $-0.0255$ \\
\texttt{all} & \texttt{positive\_gap\_rate} & $0.0297$ & $0.0169$ & $-0.0128$ \\
\texttt{all} & \texttt{tau\_mean} & $0.7388$ & $0.8029$ & $+0.0641$ \\
\texttt{eval\_rollouts} & \texttt{tau\_lt\_05\_rate} & $0.0638$ & $0.0149$ & $-0.0489$ \\
\texttt{eval\_rollouts} & \texttt{best\_agree\_rate} & $0.9754$ & $0.9969$ & $+0.0215$ \\
\texttt{eval\_rollouts} & \texttt{mean\_beta\_true\_gap} & $0.0491$ & $0.0062$ & $-0.0429$ \\
\texttt{eval\_rollouts} & \texttt{positive\_gap\_rate} & $0.0246$ & $0.0031$ & $-0.0215$ \\
\texttt{eval\_rollouts} & \texttt{tau\_mean} & $0.7672$ & $0.8271$ & $+0.0599$ \\
\bottomrule
\end{tabular}
\end{table}

\subsection{Subdivision verification table}%
\label{app:tables-subdivision}%

Table~\ref{tab:subdivision-grid} records the script-verified
subdivision sweep used by Example~\ref{ex:3.X} and
Corollary~\ref{cor:subdivision-equivariance}.  The true greedy branch from
coarse state $4$ is branch $3$ for all $k$.  Harmonic chooses branch
$1$ for all tested subdivisions, while \AMLE\ chooses branch $3$.

\begin{table}[!htbp]
\centering
\caption{Subdivision-stability verification on the seven-node
counterexample.  Margins are the method-specific first-step value gap
favoring the listed branch; all verification flags are true in
the archived subdivision-verification summary.}%
\label{tab:subdivision-grid}%
\small
\begin{tabular}{crrrr}
\toprule
subdivision $k$ & true branch & harmonic branch / margin & \AMLE\ branch / margin & verified \\
\midrule
$1$ & $3$ & $1$ / $0.1034$ & $3$ / $0.3333$ & yes \\
$2$ & $3$ & $1$ / $0.1034$ & $3$ / $0.3333$ & yes \\
$4$ & $3$ & $1$ / $0.1034$ & $3$ / $0.3333$ & yes \\
$8$ & $3$ & $1$ / $0.1034$ & $3$ / $0.3333$ & yes \\
$16$ & $3$ & $1$ / $0.1034$ & $3$ / $0.3333$ & yes \\
\bottomrule
\end{tabular}
\end{table}

The subdivision-verification table supports only the fixed-graph refinement-stability
statement in Example~\ref{ex:3.X} /
Corollary~\ref{cor:subdivision-equivariance}.  It is not used to claim that
the main rollout experiment's changing-resolution AntMaze sweep is a theorem consequence of
the seven-node example.

\subsection{Finite-\texorpdfstring{$p$}{p} Laplacian sweep}%
\label{app:tables-p-sweep}%

Table~\ref{tab:p-sweep-by-r} aggregates the finite-$p$ family sweep by resolution
$r$, supporting the §\ref{sec:exp-family-sweep} finite-$p$ baseline
discussion.  Each $(r, p)$ cell aggregates two mazes, four label
fractions, and five seeds ($40$ configurations); the $p = 2$ and
$p = \infty$ columns are main rollout references (harmonic and \AMLE\ from
Table~\ref{tab:rollout-medium-grid}--\ref{tab:rollout-large-grid} aggregated
identically).

\begin{table}[!htbp]
\centering
\caption{Finite-$p$ graph Laplacian success aggregated by
resolution $r$.  Means $\pm$ standard deviations over the
$40$-configuration aggregate at each $r$.  $p = 2$ and $p = \infty$
columns are main rollout references.}%
\label{tab:p-sweep-by-r}%
\small
\begin{tabular}{crrrrr}
\toprule
$r$ & $p = 2$ (harm.\ ref.) & $p = 4$ & $p = 8$ & $p = 16$ & $p = \infty$ (\AMLE\ ref.) \\
\midrule
$4$  & $0.850$ & $0.977 \pm 0.023$ & $0.973 \pm 0.024$ & $0.972 \pm 0.025$ & $0.993$ \\
$8$  & $0.534$ & $0.890 \pm 0.100$ & $0.980 \pm 0.031$ & $0.987 \pm 0.016$ & $0.971$ \\
$12$ & $0.368$ & $0.841 \pm 0.126$ & $0.967 \pm 0.046$ & $0.987 \pm 0.012$ & $0.946$ \\
\bottomrule
\end{tabular}
\end{table}

L-BFGS-B optimizer convergence flags (scipy \texttt{success} on the
$p$-Dirichlet energy minimisation) per $p$ across all $120$ cells
are: $p = 4$: $95 / 120$ converged ($\text{mean }\|\nabla E_p\|_\infty
= 0.0016$); $p = 8$: $10 / 120$ ($0.0082$); $p = 16$: $0 / 120$
($0.023$).  Mean per-cell max midrange-style $p$-residual at termination
(maximum across cells in parentheses):
$p = 4$: $0.04$ ($0.07$); $p = 8$: $0.29$ ($0.40$); $p = 16$:
$0.61$ ($0.77$).  These flags
indicate the energy-minimisation routine did not reach scipy's
internal tolerance at high $p$ within the iteration budget,
although the resulting surrogates still produce high greedy
success.  See §\ref{sec:exp-family-sweep} for interpretation.

\subsection{Reference baselines: nearest-label and full-graph Dijkstra}%
\label{app:tables-baselines}%

Table~\ref{tab:baselines-by-r} reports the reference baselines aggregated by
resolution $r$, comparing sparse-label \AMLE\ (main rollout reference)
against a nearest-label Voronoi assignment (using the same sparse
$\Gamma_g$ as \AMLE) and a full-graph oracle Dijkstra.  Each
$(r, \text{method})$ cell aggregates the same
$40$-configuration grouping used in Table~\ref{tab:p-sweep-by-r}.

\begin{table}[!htbp]
\centering
\caption{Reference baselines aggregated by resolution $r$.
Nearest-label uses the same sparse boundary as \AMLE; oracle
Dijkstra uses full-graph shortest-path information.}%
\label{tab:baselines-by-r}%
\small
\begin{tabular}{crrr}
\toprule
$r$ & \AMLE\ (sparse, main rollout) & nearest-label & oracle Dijkstra \\
\midrule
$4$  & $0.993$ & $0.038 \pm 0.011$ & $1.000 \pm 0.000$ \\
$8$  & $0.971$ & $0.031 \pm 0.009$ & $1.000 \pm 0.000$ \\
$12$ & $0.946$ & $0.027 \pm 0.009$ & $1.000 \pm 0.000$ \\
\bottomrule
\end{tabular}
\end{table}

The sparse-label \AMLE-vs-oracle gap is $0.7$~pp at $r = 4$,
$2.9$~pp at $r = 8$, and $5.4$~pp at $r = 12$; the gap grows with
graph size but remains small relative to the $+93.9$~pp aggregate
margin \AMLE\ holds over nearest-label.

\subsection{\AMLE\ iteration audit}%
\label{app:tables-amle-iter}%

Table~\ref{tab:amle-iteration} reports the \AMLE\ iteration audit:
on $24$ representative cells (two mazes, $r \in \{8, 12\}$,
$\mathrm{lf} \in \{0.02, 0.08\}$, three seeds), \AMLE\ is rerun at
seven iteration budgets with frozen random state across budgets
(paired $\Gamma_g$ and start--goal pairs).  Each row aggregates
$24$ cell-level success / loop / residual values.

\begin{table}[!htbp]
\centering
\caption{\AMLE\ iteration audit on a $24$-cell
representative subset; means $\pm$ standard deviations over the
$24$ cells.  $\| R_\infty \|_\infty$ is the maximum midrange
residual at solver termination, averaged over cells.}%
\label{tab:amle-iteration}%
\small
\begin{tabular}{crrrr}
\toprule
iterations & success & loop share & mean $\| R_\infty \|_\infty$ & max $\| R_\infty \|_\infty$ \\
\midrule
$50$    & $0.707 \pm 0.196$ & $0.293$ & $5.5 \times 10^{-2}$ & $1.1 \times 10^{-1}$ \\
$100$   & $0.808 \pm 0.166$ & $0.192$ & $3.2 \times 10^{-2}$ & $6.8 \times 10^{-2}$ \\
$200$   & $0.894 \pm 0.116$ & $0.106$ & $1.7 \times 10^{-2}$ & $3.4 \times 10^{-2}$ \\
$500$   & $0.970 \pm 0.049$ & $0.030$ & $6.2 \times 10^{-3}$ & $1.7 \times 10^{-2}$ \\
$1000$  & $0.996 \pm 0.014$ & $0.004$ & $2.3 \times 10^{-3}$ & $9.5 \times 10^{-3}$ \\
$2000$  & $1.000 \pm 0.002$ & $0.000$ & $7.3 \times 10^{-4}$ & $3.0 \times 10^{-3}$ \\
$5000$  & $1.000 \pm 0.000$ & $0.000$ & $9.2 \times 10^{-5}$ & $3.8 \times 10^{-4}$ \\
\bottomrule
\end{tabular}
\end{table}

The $95\%$-of-converged-success stabilisation threshold sits at
iteration $500$ (median across cells), matching the approximate
iteration budget used by the main rollout runner.  The main
rollout \AMLE\
aggregate $0.970$ is therefore consistent with a partially-
converged solver; longer iteration drives the observed rollout
failure rate to zero on this $24$-cell subset.

\subsection{Finite-\texorpdfstring{$p$}{p} solver audit}%
\label{app:tables-solver-audit}%

Table~\ref{tab:solver-audit} reports the solver audit
supporting the §\ref{sec:exp-family-sweep} finite-$p$ caveat.  The
$24$-cell representative subset (mazes $\in$ \{medium, large\},
$r \in \{8, 12\}$, $\mathrm{lf} \in \{0.02, 0.08\}$, seeds
$\in \{54, 55, 56\}$) is solved at each $p \in \{4, 8, 16\}$ by
three solvers: L-BFGS-B with a $50{,}000$-iteration cap and
$\nabla E_p$ tolerance $10^{-6}$; fixed-point Picard on the
$p$-Laplacian update at $5000$ outer sweeps with relaxation
$\omega = 0.05$ and update-tolerance $10^{-6}$; and Newton-CG with
a $20$-iteration cap (an intentionally tight budget chosen as a
curvature-stress diagnostic, not a competitive solver).

\begin{table}[!htbp]
\centering
\caption{Finite-$p$ solver audit on a $24$-cell representative
subset.  ``conv.\@'' is the count of cells the solver reports as
converged ($n = 24$).  $\| R_p \|_\infty$ is the $L^\infty$ norm of
the $p$-Laplacian update residual at termination, averaged over
cells.  Newton-CG with a $20$-iteration cap is included as a
curvature-stress diagnostic and is not a competitive solver at
high $p$.}%
\label{tab:solver-audit}%
\small
\begin{tabular}{cl rrr}
\toprule
$p$ & solver & success & conv.\@ & mean $\| R_p \|_\infty$ \\
\midrule
$4$  & L-BFGS-B ($50{,}000$ cap)    & $0.829 \pm 0.125$ & $24 / 24$ & $0.019$ \\
$4$  & fixed-point ($5000$ sweeps)  & $0.827 \pm 0.127$ & $10 / 24$ & $0.003$ \\
$4$  & Newton-CG ($20$ cap)         & $0.829 \pm 0.125$ & $5 / 24$  & $0.036$ \\
\midrule
$8$  & L-BFGS-B ($50{,}000$ cap)    & $0.958 \pm 0.049$ & $23 / 24$ & $0.178$ \\
$8$  & fixed-point ($5000$ sweeps)  & $0.965 \pm 0.047$ & $11 / 24$ & $0.002$ \\
$8$  & Newton-CG ($20$ cap)         & $0.908 \pm 0.105$ & $0 / 24$  & $0.834$ \\
\midrule
$16$ & L-BFGS-B ($50{,}000$ cap)    & $0.982 \pm 0.015$ & $16 / 24$ & $0.436$ \\
$16$ & fixed-point ($5000$ sweeps)  & $0.997 \pm 0.004$ & $15 / 24$ & $0.00013$ \\
$16$ & Newton-CG ($20$ cap)         & $0.551 \pm 0.219$ & $0 / 24$  & $1.646$ \\
\bottomrule
\end{tabular}
\end{table}

The two well-behaved solvers (L-BFGS-B at the $50{,}000$ cap and
fixed-point at $5000$ sweeps) report similar success on
$p \in \{4, 8\}$ and qualitatively reproduce the finite-$p$
family sweep's high $p = 16$ success.  They disagree quantitatively at $p = 16$
in a revealing way: L-BFGS-B reports gradient-converged solutions
($\|\nabla E_p\|_\infty \sim 5 \cdot 10^{-4}$) whose $p$-Laplacian
update residual is still order $0.4$, while fixed-point drives
the update residual to $\sim 10^{-4}$ but does not certify the
gradient.  Newton-CG with a $20$-iteration cap fails to converge
anywhere at $p \in \{8, 16\}$, yielding success values that are
solver artefacts.  Per the §\ref{sec:exp-family-sweep} interpretation,
this cross-solver disagreement at $p = 16$ means the finite-$p$ family sweep lift of
$p = 16$ over $p = \infty$ cannot be read as a planner-admissibility
advantage of finite $p$; rather, the $p = \infty$ endpoint is the
principled choice because its local update rule does not suffer
the high-$p$ ill-conditioning that prevents a clean comparison
across solvers.

\FloatBarrier
\subsection{Operator-compatibility mechanism audit}%
\label{app:tables-mechanism}%

The mechanism audit recomputes the ordering-audit rollout-weighted
decision scopes on the full $120$-cell main rollout grid and
evaluates the true shortest-path distance
$d_g$ under two local operators (the operator-level objects of
Corollary~\ref{cor:harmonic-residual} and
Proposition~\ref{prop:3.7}).  The local geometry class is
\AMLE-compatible but harmonic-incompatible when the \AMLE\ 
midrange residual $\mathcal{A}[d_g](x) - d_g(x)$ vanishes while
the harmonic averaging residual
$\frac{1}{\deg(x)}\sum_{y \sim x} d_g(y) - d_g(x)$ is non-zero
over the greedy neighbor scope, i.e.\ the $n_+(x) \ne n_-(x)$
regime of Corollary~\ref{cor:harmonic-residual}.  Harmonic inversions are
counted only when the true-best neighbor is non-tied and harmonic
strictly ranks a worse neighbor ahead of it; \AMLE\ correction
means that \AMLE\ selects a true-best neighbor on that same
decision.  The certified-correction column is the single-state \AMLE\
local-admissibility test from Lemma~\ref{lem:3.6.X.2} firing,
i.e.\ $\epsilon_s^{\AMLE} < \Delta^\star_s / 2$ at the
inverted state (with \AMLE\ error bounded as in
Corollary~\ref{cor:amle-admissibility}), expressed as a fraction
of harmonic-inverted states.  Harmonic anti-admissibility
is the Lemma~\ref{lem:harmonic-anti-admissibility} certificate
and is satisfied by definition on inversion events.

\begin{table}[!htbp]
\centering
\caption{Aggregate operator-compatibility and inversion audit.
Geometry rows are over all rollout-weighted decisions
($2{,}481{,}029$ total).  Inversion and correction rows are
conditioned on $534{,}916$ non-tied true-best decisions, except
where the denominator is explicitly the harmonic-inversion count.}%
\label{tab:mechanism-aggregate}%
\small
\begin{tabular}{p{0.57\textwidth}rr}
\toprule
diagnostic & count / denominator & rate \\
\midrule
both-compatible geometry & $1{,}079{,}361 / 2{,}481{,}029$ & $43.5\%$ \\
\AMLE-compatible / harmonic-incompatible geometry & $1{,}358{,}673 / 2{,}481{,}029$ & $54.8\%$ \\
both-incompatible geometry & $42{,}995 / 2{,}481{,}029$ & $1.7\%$ \\
harmonic inversions & $30{,}037 / 534{,}916$ & $5.6\%$ \\
inversions in primary geometry & $29{,}821 / 30{,}037$ & $99.3\%$ \\
\AMLE\ corrections of harmonic inversions & $27{,}977 / 30{,}037$ & $93.1\%$ \\
certified corrections of harmonic inversions & $20{,}135 / 30{,}037$ & $67.0\%$ \\
\bottomrule
\end{tabular}
\end{table}

\begin{table}[!htbp]
\centering
\caption{Mechanism audit by graph resolution $r$.  All entries
are percentages; correction and certificate rates are conditional on
harmonic inversions at the given resolution.}%
\label{tab:mechanism-by-r}%
\small
\begin{tabular}{crrrr}
\toprule
$r$ & primary geometry & harmonic inversion & \AMLE\ correction & certified correction \\
\midrule
$4$  & $84.4$ & $4.7$ & $88.4$ & $39.8$ \\
$8$  & $56.6$ & $5.5$ & $94.8$ & $71.7$ \\
$12$ & $38.5$ & $6.5$ & $94.9$ & $81.6$ \\
\bottomrule
\end{tabular}
\end{table}

The primary \AMLE-compatible / harmonic-incompatible geometry is most
common at the coarsest resolution, but harmonic inversions remain
present at all resolutions and become more often certified by the
local-gap condition as $r$ increases.  The audit is deliberately
local and distributional: it supports the operator-mismatch reading
of the AntMaze rollouts, not universal \AMLE\ dominance.

\FloatBarrier
\subsection{\texorpdfstring{$p$}{p}-family ordering summary}%
\label{app:tables-p-family}%

The $p$-family ordering summary is summary-only: it merges the
ordering-audit local-ordering diagnostics for $p = 2$ and
$p = \infty$ with the finite-$p$ family-sweep summaries
for $p \in \{4, 8, 16\}$ on the same $120$ case keys.
It does not compute finite-$p$ local-gap certificates, and the
$p = 16$ rows retain the finite-$p$ sweep caveat that L-BFGS-B
was run at the $700$-iteration budget without scipy optimizer
convergence.

\begin{table}[!htbp]
\centering
\caption{Summary-only $p$-family ordering table over the
$120$-cell main rollout grid.  The local-ordering columns are
scope-normalized ordering-audit / finite-$p$ sweep diagnostics;
low-$\taunbr$ is
\texttt{tau\_lt\_05\_rate}.}%
\label{tab:p-family-summary}%
\small
\begin{tabular}{crrrrrr}
\toprule
$p$ & success & loop & low-$\taunbr$ & best agree & mean gap & mean $\taunbr$ \\
\midrule
$2$        & $0.584$ & $0.416$ & $0.081$ & $0.973$ & $0.054$ & $0.753$ \\
$4$        & $0.903$ & $0.097$ & $0.051$ & $0.987$ & $0.026$ & $0.775$ \\
$8$        & $0.973$ & $0.027$ & $0.044$ & $0.987$ & $0.025$ & $0.777$ \\
$16$       & $0.982$ & $0.018$ & $0.043$ & $0.986$ & $0.027$ & $0.776$ \\
$\infty$  & $0.970$ & $0.030$ & $0.045$ & $0.990$ & $0.020$ & $0.815$ \\
\bottomrule
\end{tabular}
\end{table}

The family summary supports the same qualitative transition as
the p-family sweep and solver audit: the high-$p$ cluster has
much lower local-ordering tail
mass than $p = 2$.  It is not monotone enough to be treated as a
finite-$p$ theorem.  Across the $120$ paired cases, success is
monotone nondecreasing in only $33 / 120$ cases, mean $\taunbr$ is
monotone nondecreasing in $17 / 120$, and $p = 16$ lies between
$p = 8$ and $p = \infty$ in success on $57 / 120$ cases.  These
failures are reported as diagnostics, not anomalies requiring
post-hoc filtering.

\FloatBarrier
\subsection{Adversarial \AMLE-versus-harmonic subgraph search}%
\label{app:tables-adversarial}%

The adversarial subgraph search performs a random search for small-graph instances on which
\AMLE\ underperforms harmonic, complementing the
Example~\ref{ex:3.X} witness in the opposite direction.  Search
candidates are random connected induced subgraphs of $4 \times 4$,
$5 \times 4$, and $5 \times 5$ unweighted lattices; for each
candidate, the boundary $\Gamma_g$ is set to the goal vertex plus
$1$--$4$ additional non-goal vertices labelled by their true
shortest-path distance, and both surrogates are run to update-norm
tolerance $10^{-8}$.  Of $588$ tested $4 \times 4$ candidates with
random seed $20260510$, $5$ exhibit harmonic-greedy success
$1.000$ versus \AMLE-greedy success $\le 0.83$.  The mechanism is
common across the five witnesses: \AMLE\ produces a high-value
plateau on a pair of interior vertices whose only nearby
labelled vertex is the high-cost boundary, and the plateau
inverts the local greedy ordering between the plateau vertices.

Table~\ref{tab:adversarial-witnesses} summarises the five witnesses.
Full edge lists, boundary values, and per-vertex surrogate values
for each are archived with the adversarial-search summary.

\begin{table}[!htbp]
\centering
\caption{Adversarial \AMLE-bad / harmonic-good witnesses on
$4 \times 4$ subgraphs.  $|V|$ is the candidate's vertex count,
goal is the labelled goal vertex, $|\partial|$ is the number of
labelled boundary vertices, and the success columns count greedy
rollouts launched from each non-boundary vertex.  Failures are
listed by interior-vertex label.}%
\label{tab:adversarial-witnesses}%
\small
\begin{tabular}{ccccccl}
\toprule
\# & $|V|$ & goal & $|\partial|$ & harm.\ succ.\ & \AMLE\ succ.\ & \AMLE\ failures \\
\midrule
$1$ & $13$ & $10$ & $5$ & $1.000$ & $0.833$ & $\{12, 13\}$ \\
$2$ & $14$ & $6$  & $4$ & $0.769$ & $0.692$ & $\{8, 9, 12, 13\}$ \\
$3$ & $13$ & $15$ & $4$ & $1.000$ & $0.833$ & $\{1, 5\}$ \\
$4$ & $16$ & $3$  & $5$ & $1.000$ & $0.800$ & $\{8, 12, 13\}$ \\
$5$ & $13$ & $2$  & $5$ & $0.750$ & $0.667$ & $\{7, 11, 14, 15\}$ \\
\bottomrule
\end{tabular}
\end{table}

Witnesses \#1, \#3, \#4 have harmonic at the perfect-success
ceiling and \AMLE\ strictly below; witnesses \#2 and \#5 show both
surrogates failing on overlapping interior sets, with harmonic
still strictly better.  These five small-graph witnesses are
existence proofs that the converse direction (\AMLE-bad,
harmonic-good) is realisable on lattice subgraphs with random
sparse boundaries; combined with Example~\ref{ex:3.X}, they
establish that neither $\AMLE$ nor harmonic universally dominates
on sparse-boundary graph planning.  The main rollout AntMaze aggregate
($+38.6$~pp paired lift in favour of \AMLE) is therefore a
\emph{distributional} statement about AntMaze graph geometry, not
a universal-dominance theorem.

\FloatBarrier


\section{Proofs of §3 statements}
\label{app:proofs}

This appendix gathers proof bodies for the theorems, lemmas, propositions, and corollaries stated in Section~\ref{sec:admissibility}.  Each is keyed by its label in the main text.

\subsection*{Statement and proof of Corollary~\ref{cor:3.6.X.3}}
\label{app:cor-3.6.X.3}

\begin{corollary}[End-to-end ordinal fidelity, restated]
\label{cor:3.6.X.3}
Combining Lemmas~\ref{lem:3.6.X.1} and~\ref{lem:3.6.X.2} with
$A = N(s)$,
\[
  \epsilon_s \;\le\; \epslab + \etainf(\Vstar; N(s), \Gamma_g)
  \;\le\; \epslab + 2 L_g \, h_{\Gamma_g}(N(s)),
\]
and, for any non-leaf state $s$ ($d(s) \ge 2$),
\[
  \taunbr^{\textsc{AMLE}}(s, g)
  \;\ge\; 1 - \frac{4 \, M_{s, g}(2\epslab + 4 L_g \, h_{\Gamma_g}(N(s)))}{d(s)(d(s) - 1)}.
\]
\end{corollary}

\begin{proof}
Immediate from Lemma~\ref{lem:3.6.X.1} applied with $A = N(s)$
(controlling $\epsilon_s$ in terms of fill distance) and
Lemma~\ref{lem:3.6.X.2} (converting $\epsilon_s$ to a
neighbour-Kendall-$\tau$ bound).
\end{proof}

\subsection*{Proof of Lemma~\ref{lem:3.6.3} (\texorpdfstring{$p$-Laplacian}{p-Laplacian} maximum principle)}
\label{app:proof-lem-3-6-3}

\begin{proof}
We prove the no-strict-minimum claim; the no-strict-maximum claim
is identical with inequalities reversed.  Suppose for contradiction
that $x \in V \setminus \Gamma_g$ is a strict interior local
minimum of $u$, i.e.\ $u(y) > u(x)$ for every $y \sim x$.

\emph{Case $p = 2$ (harmonic).}  Strict minimality gives
$u(y) - u(x) > 0$ for every $y \sim x$, so
$\sum_{y \sim x}(u(y) - u(x)) > 0$, contradicting the harmonic
equation $\sum_{y \sim x}(u(y) - u(x)) = 0$.

\emph{Case $2 < p < \infty$.}  Strict minimality gives both
$|u(y) - u(x)| > 0$ and $u(y) - u(x) > 0$ for every $y \sim x$, so
$|u(y) - u(x)|^{p - 2}(u(y) - u(x)) > 0$.  Hence
$\sum_{y \sim x} |u(y) - u(x)|^{p - 2}(u(y) - u(x)) > 0$,
contradicting $\Delta_p u(x) = 0$.

\emph{Case $p = \infty$ (\AMLE).}  Strict minimality gives
$\min_{y \sim x} u(y) > u(x)$ and $\max_{y \sim x} u(y) > u(x)$,
so the midrange satisfies
$\tfrac{1}{2}(\min_y u(y) + \max_y u(y)) > u(x)$, contradicting
the \AMLE\ identity $u(x) = \tfrac{1}{2}(\min_y u(y) + \max_y u(y))$.
\end{proof}

\subsection*{Statement and proof of Proposition~\ref{prop:local-separation} (combined local separation)}
\label{app:proof-prop-local-separation}

\begin{proposition}[Combined \AMLE-vs-harmonic local separation]
\label{prop:local-separation}
Fix a decision state $s$ with true optimal neighbour set
$A^\star(s, g)$.  Suppose
\begin{enumerate}[label=(\roman*), itemsep=0pt, topsep=0pt]
\item (\emph{\AMLE\ local admissibility,
Lemma~\ref{lem:3.6.X.2}}) $\epsilon_s^{\AMLE} := \max_{y \sim s}
|\hat V_\infty(y, g) - \Vstar(y, g)| < \Delta^\star_s / 2$, and
\item (\emph{harmonic local anti-admissibility,
Lemma~\ref{lem:harmonic-anti-admissibility}}) there exists
$b \in N(s) \setminus A^\star(s, g)$ with
$Q_{\hat u_2, g}(s, b) < \min_{a \in A^\star(s, g)}
Q_{\hat u_2, g}(s, a)$, equivalently
$w(s, b) - w(s, a) + \sum_{z \in \Gamma_g}
(\omega_b(z) - \omega_a(z)) Y_g(z) < 0$ for every
$a \in A^\star(s, g)$.
\end{enumerate}
Then $T_{\hat V_\infty, g}(s) \in A^\star(s, g)$ (\AMLE-greedy
correct) and $T_{\hat u_2, g}(s) \notin A^\star(s, g)$
(harmonic-greedy wrong).
\end{proposition}

\begin{proof}
Clause (i) and Lemma~\ref{lem:3.6.X.2} give $\arg\min_y Q_{\hat
V_\infty, g}(s, y) \subseteq A^\star(s, g)$.  Clause (ii) gives
$\min_y Q_{\hat u_2, g}(s, y) \le Q_{\hat u_2, g}(s, b) <
\min_{a \in A^\star} Q_{\hat u_2, g}(s, a)$, so
$T_{\hat u_2, g}(s) \notin A^\star(s, g)$.
\end{proof}

\subsection*{Statement and proof of Lemma~\ref{lem:3.6.3.B} (residual--margin bound)}
\label{app:proof-lem-3-6-3-B}

For a surrogate $\hat u : V \to \mathbb{R}$ with
$\hat u\rvert_{\Gamma_g} = Y_g$, write $R_\infty(x; \hat u) :=
\mathcal{A}[\hat u](x) - \hat u(x)$ for the midrange residual, and
$\gamma(s^\ast) := \min_{s' \sim s^\ast} \hat u(s') - \hat
u(s^\ast)$ for the strict-min depth at any approximate strict
interior local minimum $s^\ast$ of $\hat u$.

\begin{lemma}[Residual--margin bound for finite-sweep \AMLE]
\label{lem:3.6.3.B}
If $\|R_\infty(\cdot; \hat u)\|_\infty \le \delta$ then
$\gamma(s^\ast) \le \delta$ at every approximate strict interior
local minimum $s^\ast$.
\end{lemma}

\begin{proof}
Let $s^\ast$ be an approximate strict interior local minimum of
$\hat u$ with $\gamma(s^\ast) > 0$, so
$\min_{s' \sim s^\ast} \hat u(s') = \hat u(s^\ast) + \gamma(s^\ast)$
and every neighbor exceeds $\hat u(s^\ast)$ by at least
$\gamma(s^\ast)$, whence
$\max_{s' \sim s^\ast} \hat u(s') \ge \hat u(s^\ast) + \gamma(s^\ast)$.
Therefore
\[
  \mathcal{A}[\hat u](s^\ast)
  \;=\; \tfrac{1}{2}\bigl(\min_{s' \sim s^\ast} \hat u(s') + \max_{s' \sim s^\ast} \hat u(s')\bigr)
  \;\ge\; \hat u(s^\ast) + \gamma(s^\ast),
\]
so $R_\infty(s^\ast; \hat u) \ge \gamma(s^\ast)$.  Combining with
$\|R_\infty\|_\infty \le \delta$ gives $\gamma(s^\ast) \le \delta$.
\end{proof}

\subsection*{Proof of Lemma~\ref{lem:3.6.X.1} (\AMLE\ local extension error)}
\label{app:proof-lem-3-6-X-1}

\begin{proof}[Proof sketch]
By the boundary-pinned \AMLE\ comparison principle on graphs
\citep{kyng2015algorithms, bungert2021lipschitz},
$\| \mathcal{A}_{\Gamma_g}(z) - \mathcal{A}_{\Gamma_g}(z') \|_\infty \le \| z - z' \|_{\infty, \Gamma_g}$.
Apply with $z = Y_g$ and $z' = \Vstar\rvert_{\Gamma_g}$ to obtain
$\| \hat V_\infty - \mathcal{A}_{\Gamma_g}(\Vstar\rvert_{\Gamma_g}) \|_\infty \le \epslab$;
triangle inequality with $\etainf$ gives the first bound.  For the
fill-distance bound, \AMLE\ preserves the Lipschitz constant of the
boundary data; for any $x \in A$ and a closest labeled
$z \in \Gamma_g$, both $\mathcal{A}_{\Gamma_g}(\Vstar\rvert_{\Gamma_g})$
and $\Vstar$ change by at most $L_g \, h_{\Gamma_g}(A)$ along a path
from $z$ to $x$, so their difference at $x$ is at most
$2 L_g \, h_{\Gamma_g}(A)$.
\end{proof}

\subsection*{Proof of Lemma~\ref{lem:3.6.X} (gap-dependent local ordinal fidelity)}
\label{app:proof-lem-3-6-X}

\begin{proof}
For any pair of neighbours $y, y' \in N(s)$,
$Q_{\Vhat, g}(s, y) - Q_{\Vhat, g}(s, y')$ differs from
$Q^\star(s, y; g) - Q^\star(s, y'; g)$ by exactly
$(\Vhat(y, g) - \Vstar(y, g)) - (\Vhat(y', g) - \Vstar(y', g))$,
which has magnitude at most $2 \epsilon_s$ (the proof uses only
this local sup error bound and not any \AMLE-specific structure).
Hence if $|Q^\star(s, y; g) - Q^\star(s, y'; g)| > 2 \epsilon_s$,
the surrogate ordering on $(y, y')$ matches the true ordering and
the pair is not an inversion.  Therefore
$I_{s, g} \le M_{s, g}(2 \epsilon_s)$, and the Kendall-tau
identity gives the displayed bound.

For the action-ordering preservation part, pick any
$a \in A^\star(s, g)$ and any
$b \in N(s) \setminus A^\star(s, g)$.  By definition of $A^\star$
and $\Delta^\star_s$, $Q^\star(s, b; g) - Q^\star(s, a; g) \ge
\Delta^\star_s > 2 \epsilon_s$; the same gap argument forces
$Q_{\Vhat, g}(s, b) > Q_{\Vhat, g}(s, a)$.  Hence
$\arg\min_y Q_{\Vhat, g}(s, y) \subseteq A^\star(s, g)$.
\end{proof}

\subsection*{Statement and proof of Corollary~\ref{cor:3.6.X.4} (bad-tail mass)}
\label{app:proof-cor-3-6-X-4}

\begin{corollary}[Bad-tail mass under local-gap regularity, restated]
\label{cor:3.6.X.4}
Let $\mu_{\mathrm{dec}}$ be the planner's decision distribution
over $(s, g)$ pairs with $d(s) \ge 2$, and define the small-gap
pair fraction $G_{s, g}(\eta) := M_{s, g}(\eta) /
\binom{d(s)}{2}$.  Suppose $G_{s, g}(\eta) \le C_0 \eta^\alpha$
uniformly in $(s, g) \sim \mu_{\mathrm{dec}}$ for some $\alpha,
C_0 > 0$ (a hypothesis on the local action-gap CDF, not a
derived property).  Then for any threshold $\theta < 1$,
\[
  \mu_{\mathrm{dec}}\bigl\{ \taunbr \le \theta \bigr\}
  \;\le\; \frac{2^{\alpha + 1} C_0}{1 - \theta} \, \mathbb{E}_{\mu_{\mathrm{dec}}}[\epsilon_s^\alpha].
\]
\end{corollary}

\begin{proof}
By Lemma~\ref{lem:3.6.X.2},
$1 - \taunbr(s, g) \le \tfrac{4}{d(d-1)} M_{s, g}(2\epsilon_s)
= 2 G_{s, g}(2 \epsilon_s) \le 2 C_0 (2\epsilon_s)^\alpha
= 2^{\alpha+1} C_0 \epsilon_s^\alpha$.
Markov inequality on $1 - \taunbr$ at level $1 - \theta$ gives the bound.
\end{proof}

\subsection*{Proof of Proposition~\ref{prop:3.7} (distance is graph-\AMLE\ on the extendable interior)}
\label{app:proof-prop-3-7}

\begin{proof}
\emph{Part (a).}  By the unit-graph triangle inequality, $|d_g(y) - d_g(x)| \le 1$
for every $y \sim x$, so $\min_{y \sim x} d_g(y) \ge d_g(x) - 1$.  Let
$x = x_0, x_1, \ldots, x_k = g$ be a shortest path with $k = d_g(x)$;
then $x_1$ is a neighbor of $x$ with $d_g(x_1) = d_g(x) - 1$, giving
the reverse inequality.

\emph{Part (b).}  By the same triangle inequality,
$\max_{y \sim x} d_g(y) \le d_g(x) + 1$, and $\max \ge \min = d_g(x) - 1$.
The case $\max = d_g(x) - 1$ does occur at non-leaf vertices in
cut-locus regions (e.g., a diamond graph: $g$ at distance $0$, two
intermediate neighbors at distance $1$, and a vertex $x$ at distance
$2$ adjacent to both intermediates --- every neighbor of $x$ is
closer to $g$).

\emph{Part (c).}  By (a), $\min_{y \sim x} d_g(y) = d_g(x) - 1$ always.
Setting $\mathcal{A}[d_g](x) = d_g(x)$ and using (b),
\[
  d_g(x) \;=\; \tfrac{1}{2}\bigl((d_g(x) - 1) + \max_{y \sim x} d_g(y)\bigr)
  \;\iff\; \max_{y \sim x} d_g(y) = d_g(x) + 1
  \;\iff\; x \in V^{\mathrm{ext}}_g.
\]
The Bellman / eikonal identity at $x \neq g$ follows directly from
(a).
\end{proof}

\subsection*{Statement and proof of Corollary~\ref{cor:subdivision-equivariance} (subdivision equivariance of strict harmonic rankings)}
\label{app:proof-cor-subdivision}

\begin{corollary}[Subdivision-equivariance of strict harmonic rankings]
\label{cor:subdivision-equivariance}
Strict harmonic neighbour rankings of the type described by
Lemma~\ref{lem:harmonic-measure} are preserved under uniform
$k$-subdivision of $G$ with label scaling $\lambda_k = k$
(metric-preserving).  In particular, the inversion of
Example~\ref{ex:3.X} persists at every $k \ge 1$.  The
analogous statement on the \AMLE\ side is
Lemma~\ref{lem:amle-subdivision} below, so the discriminator
between the two endpoints lies in
Lemma~\ref{lem:harmonic-anti-admissibility} vs
Corollary~\ref{cor:amle-admissibility}, not in
subdivision-stability.
\end{corollary}

\begin{proof}
Define $\tilde h^{(k)}$ on $G^{(k)}$ by $\tilde h^{(k)}(v) =
\lambda_k h(v)$ at every original vertex and by linear
interpolation along subdivided edges.  At an inserted vertex
$v_j$ of an edge $v \to z$ (with $1 \le j \le k - 1$), the two
$G^{(k)}$-neighbours are $v_{j - 1}$ and $v_{j + 1}$, and linear
interpolation gives $\tilde h^{(k)}(v_j) = (\tilde h^{(k)}(v_{j - 1}) +
\tilde h^{(k)}(v_{j + 1}))/2$, i.e.\ the harmonic averaging
identity is satisfied on the inserted path.  At an original
vertex $v \in V \setminus \bdry$, the $G^{(k)}$-neighbours are
the first inserted vertices $v_1$ on each edge $v \to z$
($z \sim v$ in $G$); by linear interpolation,
$\tilde h^{(k)}(v_1) = \lambda_k h(v) + (\lambda_k h(z) -
\lambda_k h(v))/k$.  Averaging over the original-graph neighbours
$z \sim v$:
\[
  \tfrac{1}{\deg(v)} \sum_{z \sim v} \tilde h^{(k)}(v_1)
  \;=\; \lambda_k h(v) + \tfrac{\lambda_k}{k \deg(v)} \sum_{z \sim v} (h(z) - h(v))
  \;=\; \lambda_k h(v),
\]
since $h$ is harmonic at $v$ in $G$.  Hence $\tilde h^{(k)}$
satisfies the harmonic averaging identity at every interior
vertex of $G^{(k)}$ with the scaled boundary data
$\lambda_k y$, and by uniqueness of the harmonic extension
$h^{(k)} = \tilde h^{(k)}$.

For the first-step branch difference at any original decision
state $s$ with original neighbours $a, b$: the
$G^{(k)}$-neighbours of $s$ are the first inserted vertices
$a^{(1)}, b^{(1)}$ along edges $\{s, a\}, \{s, b\}$ respectively,
with $\tilde h^{(k)}(a^{(1)}) = \lambda_k h(s) + \lambda_k (h(a)
- h(s))/k$ and analogously for $b^{(1)}$.  The Q-value difference
(under unit edge cost in $G^{(k)}$) is
$\hat Q^{(k)}_{a^{(1)}} - \hat Q^{(k)}_{b^{(1)}}
= \tilde h^{(k)}(a^{(1)}) - \tilde h^{(k)}(b^{(1)})
= (\lambda_k / k)(h(a) - h(b))$.  Taking $\lambda_k = k$ gives
the $k$-independent margin $h(a) - h(b)$.
\end{proof}

\subsection*{Statement and proof of Lemma~\ref{lem:amle-subdivision} (\AMLE\ midrange subdivision equivariance)}
\label{app:lem-amle-subdivision}

\begin{lemma}[\AMLE\ midrange equivariance under uniform subdivision]
\label{lem:amle-subdivision}
Let $u_\infty$ be the exact graph-\AMLE\ on $G$ with boundary
$\bdry$ and labels $y$.  Uniformly subdivide each edge into a
length-$k$ path and scale boundary labels by $\lambda_k > 0$.
The linear interpolation of $\lambda_k u_\infty$ along each
subdivided edge satisfies the midrange equation on the
subdivided graph $G^{(k)}$, so $u_\infty^{(k)}(v) =
\lambda_k u_\infty(v)$ at every original vertex.  Consequently,
for original neighbour candidates $a, b$ of an original decision
state $s$, the first-step \AMLE\ branch difference satisfies
$\hat Q^{(k), \textsc{AMLE}}_{a^{(1)}} - \hat Q^{(k), \textsc{AMLE}}_{b^{(1)}}
= (\lambda_k/k)(u_\infty(a) - u_\infty(b))$; with
$\lambda_k = k$ the margin is independent of $k$ and any strict
\AMLE\ neighbour ranking on $G$ persists for every $k \ge 1$.
\end{lemma}

\begin{proof}
Linear interpolation along a degree-2 inserted path satisfies the
midrange identity automatically (with neighbours of an inserted
vertex at adjacent path-positions, $\min$ and $\max$ are the two
neighbour values, so the midrange equals their average, which
equals the linearly-interpolated value).  At an original vertex
$v$, the $G^{(k)}$-neighbours are the first inserted vertices
$v_1$ along each edge $v \to z$, with values
$\lambda_k u_\infty(v) + \lambda_k (u_\infty(z) -
u_\infty(v))/k$.  The midrange of these values over $z \sim v$
equals $\lambda_k u_\infty(v) + (\lambda_k / k) \cdot
\tfrac{1}{2}(\min_z u_\infty(z) - u_\infty(v) + \max_z u_\infty(z) - u_\infty(v))$,
which equals $\lambda_k u_\infty(v)$ exactly when $u_\infty$
satisfies the midrange identity at $v$ in $G$.  Hence
$u_\infty^{(k)} = \lambda_k u_\infty$ on original vertices by
uniqueness of the \AMLE\ extension.  The branch-difference formula
follows by the same first-inserted-vertex computation as in the
harmonic case.
\end{proof}

\subsection*{The seven-node graph \texorpdfstring{$G_7$}{G7} (Example~\ref{ex:3.X})}
\label{app:g7-setup}

Define $G_7 = (V_7, E_7)$ by
\[
  V_7 = \{0, 1, 3, 4, 5, 6, 7\},
  \qquad
  E_7 = \bigl\{
    \{0, 3\}, \{0, 5\}, \{5, 1\}, \{1, 4\},
    \{3, 4\}, \{3, 6\}, \{6, 7\}, \{4, 7\}
  \bigr\}.
\]
The goal is $g = 0$ and the sparse-label boundary is
$\Gamma_g = \{0, 7\}$ with observed labels $Y_g(0) = 0$ and
$Y_g(7) = 3$.  Shortest-path distances to $g$ are
$\Vstar(3, g) = \Vstar(5, g) = 1$ and
$\Vstar(1, g) = \Vstar(4, g) = \Vstar(6, g) = 2$, so the
decision state $s = 4$ has neighbour set $N(4) = \{1, 3, 7\}$,
true Bellman-optimal neighbour set $A^\star(4, g) = \{3\}$, and
action gap $\Delta^\star_4 = 1$.

\paragraph{Harmonic Dirichlet solve.}
Solving $\Delta_2 \hat u_2 = 0$ on $V_7 \setminus \Gamma_g$ with
$\hat u_2\rvert_{\Gamma_g} = Y_g$ yields, at the
decision-relevant neighbours of $s = 4$,
\[
  \hat u_2(1) = \tfrac{36}{29},
  \qquad
  \hat u_2(3) = \tfrac{39}{29}.
\]
Via Lemma~\ref{lem:harmonic-measure} the harmonic measures of
neighbours $1, 3$ on the non-goal boundary $\{7\}$ are
$\omega_1(7) = 12/29$ and $\omega_3(7) = 13/29$, so
$\omega_1(7) - \omega_3(7) = -1/29 < 0$.  The harmonic-greedy
planner therefore prefers $1$ over $3$ at $s = 4$, instantiating
clause (ii) of Proposition~\ref{prop:local-separation}.

\paragraph{\AMLE\ midrange solve.}
The midrange fixed point $\hat V_\infty = \mathcal{A}_{\Gamma_g}(Y_g)$
takes the values $\hat V_\infty(3) = 1$ and
$\hat V_\infty(1) = 4/3$ at the decision-relevant neighbours.
In particular $\hat V_\infty(3) = \Vstar(3, g) = 1$ exactly:
$3 \in V^{\mathrm{ext}}_g$ since the neighbour $4$ has
$\Vstar(4, g) = 2 = \Vstar(3, g) + 1$, so
Proposition~\ref{prop:3.7} applies and the midrange identity
$\hat V_\infty(3) = \mathcal{A}[\hat V_\infty](3)$ holds at the
true value.  The \AMLE-greedy planner picks $3$ correctly.

\paragraph{Mechanism-scope remark.}
At $s = 4$, the local sup errors are
$\epsilon_4^{\mathrm{harm}} = 22/29 \approx 0.76$ and
$\epsilon_4^{\AMLE} = 2/3 \approx 0.67$; both exceed
$\Delta^\star_4 / 2 = 0.5$, so the
Lemma~\ref{lem:3.6.X.2} sufficient condition fails at this
state for both surrogates.  Harmonic nevertheless mis-ranks
because its error symmetrises across the two non-goal
neighbours of $s = 4$ (the harmonic-measure averaging effect of
Lemma~\ref{lem:harmonic-measure}); \AMLE\ remains greedy-correct
through the operator-level exact match
$\hat V_\infty(3) = \Vstar(3, g) = 1$ of Proposition~\ref{prop:3.7}
instead.  $G_7$ illustrates the harmonic side of
Proposition~\ref{prop:local-separation} with positive margin,
not clause (i): clause (i)'s \AMLE-correctness hypothesis is
sufficient but not necessary.

\smallskip
\noindent\emph{Specialisation to $G_7$ (subdivision).}  Applying
the harmonic case of
Corollary~\ref{cor:subdivision-equivariance} to the $G_7$ data
($h(1) - h(3) = -3/29$ with $\lambda_k = k$) gives wrong-ordering
margin $-3/29$ at every $k$; applying the \AMLE\ case
(Lemma~\ref{lem:amle-subdivision}) to
$u_\infty(1) - u_\infty(3) = 1/3$ gives correct-ordering margin
$1/3$ at every $k$.  These are the $G_7$-specific numbers used
in §\ref{sec:harmonic-measure}.


\end{document}